\documentclass[11pt]{article}
\usepackage{amsmath}
\usepackage{amssymb}
\usepackage{mathrsfs}
\usepackage{algorithm}
\usepackage{algpseudocode}
\usepackage{amsthm}
\usepackage[framemethod=TikZ]{mdframed}
\usepackage{framed}
\usepackage[margin=1in]{geometry}
\usepackage{graphicx}
\usepackage[outdir=./]{epstopdf}
\usepackage{caption}
\usepackage{subcaption}
\usepackage{multirow}
\usepackage{comment}
\usepackage{color}
\usepackage{xcolor}
\definecolor{mydarkblue}{rgb}{0,0.08,0.45}
\usepackage[hidelinks,colorlinks]{hyperref}
\hypersetup{
	colorlinks=true,
	linkcolor=mydarkblue,
	citecolor=mydarkblue,
	filecolor=mydarkblue,
	urlcolor=mydarkblue,
}
\usepackage{thmtools,thm-restate}
\usepackage[numbers]{natbib}
\date{\vspace{-5ex}}

\allowdisplaybreaks

\newtheorem{corollary}{Corollary}
\newtheorem{lemma}{Lemma}

\newtheorem{fact}{Fact}

\theoremstyle{definition}
\newtheorem{definition}{Definition}
\theoremstyle{remark}

\newtheorem{example}{Example}
\numberwithin{equation}{section}

\let\Pr\relax
\newcommand*{\Pr}{Pr}
\newcommand*{\E}{\mathbb{E}}
\newcommand*{\poly}[1]{\text{poly}(#1)}

\newcommand{\norm}[2]{\| #1 \|_{#2}}

\newcommand\eps\epsilon

\begin{document}
%	\maketitle
%\input{abstract}
\newcommand{\theTitle}{Prediction with a Short Memory}
\author{
 \fontsize{11}{13}\selectfont {\bf Vatsal Sharan} \\
%\fontsize{11}{13}\selectfont  Computer Science Department \\
 \fontsize{11}{13}\selectfont Stanford University \\
 \fontsize{11}{13}\selectfont {\tt vsharan@stanford.edu}
\and
 \fontsize{11}{13}\selectfont {\bf Sham Kakade} \\
%\fontsize{11}{13}\selectfont  Computer Science Department \\
 \fontsize{11}{13}\selectfont University of Washington \\
 \fontsize{11}{13}\selectfont {\tt sham@cs.washington.edu}
\and
 \fontsize{11}{13}\selectfont {\bf Percy Liang} \\
%\fontsize{11}{13}\selectfont  Computer Science Department \\
 \fontsize{11}{13}\selectfont Stanford University \\
 \fontsize{11}{13}\selectfont {\tt pliang@cs.stanford.edu}
\and
\fontsize{11}{13}\selectfont {\bf Gregory Valiant} \\
%\fontsize{11}{13}\selectfont  Computer Science Department \\
\fontsize{11}{13}\selectfont Stanford University \\
\fontsize{11}{13}\selectfont {\tt valiant@stanford.edu}
}

\title{\theTitle}
\date{}

\clearpage
\maketitle\begin{abstract}
We consider the problem of predicting the next observation given a sequence of past observations, and consider the extent to which accurate prediction requires complex algorithms that explicitly leverage long-range dependencies.  Perhaps surprisingly, our positive results show that for a broad class of sequences, there is an algorithm that predicts well on average, and bases its predictions only on the most recent few observation together with a set of simple summary statistics of the past observations.  Specifically, we show that for any distribution over observations, if the mutual information between past observations and future observations is upper bounded by $I$, then a simple Markov model over the most recent $I/\epsilon$ observations obtains expected KL error $\epsilon$---and hence $\ell_1$ error $\sqrt{\epsilon}$---with respect to the optimal predictor that has access to the entire past and knows the data generating distribution. For a Hidden Markov Model with $n$ hidden states, $I$ is bounded by $\log n$, a quantity that does not depend on the mixing time, and we show that the trivial prediction algorithm based on the empirical frequencies of length $O(\log n/\epsilon)$ windows of observations achieves this error, provided the length of the sequence is $d^{\Omega(\log n/\epsilon)}$, where $d$ is the size of the observation alphabet. 

We also establish that this result cannot be improved upon, even for the class of HMMs, in the following two senses: First, for HMMs with $n$ hidden states, a window length of $\log n/\epsilon$ is information-theoretically necessary to achieve expected KL error $\epsilon$, or $\ell_1$ error $\sqrt{\epsilon}$. Second, the $d^{\Theta(\log n/\epsilon)}$ samples required to accurately estimate the Markov model when observations are drawn from an alphabet of size $d$ is necessary for any computationally tractable learning/prediction algorithm, assuming the hardness of strongly refuting a certain class of CSPs.  

%Our positive results show that accurate prediction, on average, can be achieved by an algorithm that ignores long-range dependencies in the observations; we hope that this sheds light on the recent efforts of the practical ML community to train models with explicit notions of memory.

% PL: not sure if should have too much philosophy in the abstract...
%The simplicity of a Markov model stands in stark contrast to 
%computationally demanding hidden-variable models,
%Our results show that even a simple baseline 

%For observations drawn from an alphabet
%of size $n$, our positive results show that one might need to learn the
%marginal distribution over all $n^{h/\epsilon}$ length sequences; for large $n$,
%this may require prohibitive amounts of data or memory to learn.
%Nevertheless, we show a computational lower bound, based on the assumed
%hardness of certain CSPs, showing that for any fixed $\epsilon$ and mutual
%information $h$, there exist Hidden Markov Models with observations drawn
%from alphabets of size $n$ and observation sequences with mutual information
%$h$, such that any prediction algorithm that runs in time $poly(n)$ requires
%at least $n^{\Theta(h/\epsilon)}$ samples to achieve error $\epsilon$.  
\end{abstract}

\thispagestyle{empty}
\newpage
\setcounter{page}{1}
	%\newpage

\section{Memory, Modeling, and Prediction}

% Problem setup, applications, central problem
We consider the problem of predicting the next observation $x_t$
given a sequence of past observations, $x_1, x_2, \dots, x_{t-1}$,
which could have complex and long-range dependencies.
This \emph{sequential prediction} problem is one of the most basic learning tasks and is encountered throughout natural language modeling, speech synthesis, financial forecasting, and a number of other domains that have a sequential or chronological element.   
The abstract problem has received much attention over the last half century
from multiple communities including TCS, machine learning, and coding theory.
The fundamental question is: \emph{How do we consolidate and reference memories about the past in order to effectively predict the future?}

Given the immense practical importance of this prediction problem, there has been an enormous effort to explore different algorithms for storing and referencing information about the sequence, which have led to the development of several popular models such as $n$-gram models and Hidden Markov Models (HMMs). Recently, there has been significant interest in \emph{recurrent neural networks} (RNNs) \citep{bengio1994learning}---which encode the past as a real vector of fixed length that is updated after every observation---and specific classes of such networks, such as Long Short-Term Memory (LSTM) networks \citep{hochreiter1997lstm,gers2000learning}.  Other recently popular models  that have explicit notions of memory include neural Turing machines \cite{graves2014neural}, memory networks \cite{weston2015memory}, differentiable neural computers \cite{graves2016hybrid}, attention-based models~\cite{bahdanau2014neural,vaswani2017attention}, etc.  These models have been quite successful (see e.g.~\cite{luong2015translation,wu2016google}); nevertheless,  consistently learning long-range dependencies, in settings such as natural language, remains an extremely active area of research.

In parallel to these efforts to design systems that explicitly use memory, there has been much effort from the neuroscience community to understand how humans and animals are able to make accurate predictions about their environment.  Many of these efforts also attempt to understand the computational mechanisms behind the formation of memories (memory ``consolidation'') and retrieval ~\cite{chen2017deciphering,chen2016uncovering,wilson1994reactivation}.
%arises naturally in natural language modeling \citep{chen96smoothing}, where the goal is to predict the next word in a document given the previous words; other applications include speech synthesis \citep{oord2016wavenet} and financial forecasting.

Despite the long history of studying sequential prediction, many fundamental questions remain:
\begin{itemize}
\vspace{-.15cm}
\item How much memory is necessary to accurately predict future
observations, and what properties of the underlying sequence determine this requirement?
%\vspace{-.2cm}
\item Must one remember significant information about the distant past or is a
short-term memory sufficient?
%\vspace{-.2cm}
\item  What is the computational complexity of accurate prediction?
%\vspace{-.2cm}
\item How do answers to the above questions depend on the metric that is used to evaluate prediction accuracy?
\end{itemize}
Aside from the intrinsic theoretical value of these questions, their answers could serve to guide the construction of effective practical prediction systems, as well as informing the discussion of the computational machinery of cognition and prediction/learning in nature.

In this work, we provide insights into the first three questions.  We begin by establishing the following proposition, which addresses the first two questions with respect to the pervasively used metric of \emph{average} prediction error:

\medskip

\noindent {\textbf{Proposition~\ref{prop:upperbnd}.}} \emph{Let $\mathcal{M}$ be any distribution over sequences with mutual
	information $I(\mathcal{M})$ between the past observations $\ldots,x_{t-2},x_{t-1}$ and future observations $x_{t},x_{t+1},\ldots$.
  The best $\ell$-th order Markov model, which makes predictions based only on the most recent $\ell$ observations, predicts the distribution of the next observation with average KL error
  $I(\mathcal{M})/\ell$ or average $\ell_1$ error $\sqrt{I(\mathcal{M})/\ell},$ with respect to the actual conditional distribution of $x_t$ given
  all past observations.}
\medskip

The ``best'' $\ell$-th order Markov model is the model which predicts $x_t$ based on the previous $\ell$ observations, $x_{t-\ell},\ldots,x_{t-1}$, according to the conditional distribution of $x_t$ given $x_{t-\ell},\ldots,x_{t-1}$ under the data generating distribution. If the output alphabet is of size $d$, then this conditional distribution can be estimated with small error given $O(d^{\ell+1})$ sequences drawn from the distribution. Without any additional assumptions on the data generating distribution beyond the bound on the mutual information, it is necessary to observe multiple sequences to make good predictions.  This is because the distribution could be highly non-stationary, and have different behaviors at different times, while still having small mutual information.  In some settings, such as the case where the data generating distribution corresponds to observations from an HMM, we will be able to accurately learn this ``best'' Markov model from a single sequence (see Theorem~\ref{thm:hmm}).

 %In the context of NLP, this corresponds to estimating the joint occurrence probabilities of words based on multiple documents from a large  

%\begin{restatable}{proposition}{upperbnd}[informal] \label{prop:upperbnd}
%  Let $\mathcal{M}$ be any distribution over sequences, with mutual
%  information $I(\mathcal{M})$ between the past observations $\ldots,x_{t-2},x_{t-1}$ and future observations $x_{t},x_{t+1},\ldots$.
%  The $\ell$-th order Markov model obtains average KL error $I(\mathcal{M})/\ell$ for predicting the distribution of the next observation, with respect to the true conditional distribution of $x_t$ given all past observations.
%\end{restatable}

% PL: talk about this thing later
%\begin{restatable}{proposition}{upperbnd}\label{prop:upperbnd}
%	For any true data-generating distribution $\mathcal{M}$ with mutual information $I(\mathcal{M})$ between past and future observations,
%  the best $\ell$-th order Markov model $\mathcal{P}_{\ell}$ obtains average KL-error,
%  $\delta_{KL}(\mathcal{P}_{\ell}) \le I(\mathcal{M})/\ell$
%  with respect to the optimal predictor with access to the infinite history.
%  Also, any predictor $\mathcal{A}_{\ell}$ with $\hat{\delta}_{KL}(\mathcal{A}_{\ell})$
%  average KL-error in estimating the conditional probabilities gets average error
%  $\delta_{KL}(\mathcal{A}_{\ell}) \le I(\mathcal{M})/\ell + \hat{\delta}_{KL}(\mathcal{A}_{\ell})$.
%\end{restatable}
%\pl{unclear why algorithm has to estimate the conditional probabilities, what $\hat\delta$ is}
The intuition behind the statement and proof of this general proposition is the following: at time $t$, we either predict accurately and are unsurprised when $x_t$ is revealed to us; or, if we predict poorly and are surprised by the value of $x_t$, then $x_t$ must contain a significant amount of information about the history of the sequence, which can then be leveraged in our subsequent predictions of $x_{t+1}$, $x_{t+2}$, etc.  In this sense, every timestep in which our prediction is `bad', we learn some information about the past.  Because the mutual information between the history of the sequence and the future is bounded by $I(\mathcal{M})$, if we were to make $I(\mathcal{M})$ consecutive bad predictions, we have captured nearly this amount of information about the history, and hence going forward, as long as the window we are using spans these observations, we should expect to predict well.

This general proposition, framed in terms of the mutual information of the past and future,
has immediate implications for a number of well-studied models of sequential data, such as Hidden Markov Models (HMMs).  For an HMM with $n$ hidden states, the mutual information of the generated sequence is trivially bounded by $\log n$, which yields the following corollary to the above proposition.  We state this proposition now, as it provides a helpful reference point in our discussion of the more general proposition.

% PL: I simplified the statement to avoid talking about stationary distribution.
% I think we don't need to give full results, just a taste.
\medskip
\begin{corollary}\label{hmm}
	Suppose observations are generated by a Hidden Markov Model with at most $n$ hidden states.
	 The best $\frac{\log n}{\eps}$-th order Markov model, which makes predictions based only on the most recent $\frac{\log n}{\eps}$ observations, predicts the distribution of the next observation with average KL error
  $\le \eps$ or $\ell_1$ error $\le \sqrt{\eps}$, with respect to the optimal predictor that knows the underlying HMM and has access to all past observations.	
\end{corollary}

In the setting where the observations are generated according to an HMM with at most $n$ hidden states, this ``best'' $\ell$-th order Markov model is easy to learn given a \emph{single} sufficiently long sequence drawn from the HMM, and corresponds to the naive ``empirical'' $\ell$-th order Markov model (i.e. $(\ell+1)$-gram model) based on the previous observations. Specifically, this is the model that, given $x_{t-\ell},x_{t-\ell+1},\ldots,x_{t-1},$ outputs the observed (empirical) distribution of the observation that has followed this length $\ell$ sequence.  (To predict what comes next in the phrase ``\ldots defer the details to the $\rule{0.3cm}{0.15mm}$'' we look at the previous occurrences of this subsequence, and predict according to the empirical frequency of the subsequent word.)  The following theorem makes this claim precise.

\begin{restatable}{theorem}{hmm}\label{thm:hmm}
Suppose observations are generated by a Hidden Markov Model with at most $n$ hidden states, and output alphabet of size $d$. For $\epsilon>1/\log^{0.25}n$ there exists a window length $\ell = O(\frac{\log n}{\eps})$ and absolute constant $c$ such that for any $T \ge d^{c \ell},$ if $t \in \{1,2,\ldots,T\}$ is chosen uniformly at random, then the expected $\ell_1$ distance between the true distribution of $x_t$ given the entire history (and knowledge of the HMM), and the distribution predicted by the naive ``empirical'' $\ell$-th order Markov model based on $x_0,\ldots,x_{t-1}$, is bounded by $\sqrt{\eps}$.\footnote{Theorem \ref{thm:hmm} does not have a guarantee on the average KL loss, such a guarantee is not possible as the KL loss as it can be unbounded, for example if there are rare characters which have not been observed so far.} %, where the expectation is with respect to the randomness in the choice of $t$ and in the data generating distribution. %This naive empirical Markov model is the model that, given $x_{t-\ell},x_{t-\ell+1},\ldots,x_{t-1},$ will predict according to the empirical distribution of observations that have followed the occurrences of this length $\ell$ string in the history $x_0,\ldots,x_t.$ 
\end{restatable}

The above theorem states that the window length necessary to predict well is independent of the mixing time of the HMM in question, and holds even if the model does not mix.  While the amount of data required to make accurate predictions using length $\ell$ windows scales exponentially in $\ell$---corresponding to the condition in the above theorem that $t$ is chosen uniformly between $0$ and $T=d^{O(\ell)}$---our lower bounds, discussed in Section~\ref{intro:LB}, argue that this exponential dependency is unavoidable.

\subsection{Interpretation of Mutual Information of Past and Future}
While the mutual information between the past observations and the future observations is an intuitive parameterization of the complexity of a distribution over sequences,
the fact that it is the \emph{right} quantity is a bit subtle.  It is tempting to hope that this mutual information is a bound on the amount of memory that would be required to store all the information about past observations that is relevant to the distribution of future observations.  This is \emph{not} the case.  Consider the following setting:~Given a joint
distribution over random variables $X_\text{past}$ and $X_\text{future}$, suppose we wish to define a
function $f$ that maps $X_\text{past}$ to a binary ``advice''/memory string $f(X_\text{past})$, possibly of variable
length, such that $X_\text{future}$ is independent of $X_\text{past}$, given $f(X_\text{past}).$
%given $A$ drawn from the appropriate marginal distribution,
%$B$ can be sampled from the correct joint distribution given \emph{only} the
%``advice'' $f(A)$---namely the conditional distribution of $B$ given $A$ is
%exactly the distribution of $B$ given $f(A)$.
%In our setting, $A$ is the sequence of past observations,$B$ is the sequence of future observations, and $f(A)$ is the information about the past sufficient to predict the future.
As is shown in~\citet{harsha2007communication},
there are joint distributions over
$(X_\text{past},  X_\text{future})$ such that even on average, the minimum length of the advice/memory string
necessary for the above task is exponential in the mutual information
$I(X_\text{past}; X_\text{future})$.  This setting can also be interpreted as a
two-player communication game where one player generates $X_\text{past}$ and the other
generates $X_\text{future}$ given limited communication (i.e. the ability to communicate $f(X_\text{past})$).\footnote{It is worth noting that if the advice/memory string $s$ is
sampled first, and then $X_\text{past}$ and $X_\text{future}$ are defined to be random functions of $s$,
then the length of $s$ \emph{can} be related to $I(X_\text{past}; X_\text{future})$
(see~\cite{harsha2007communication}).  This latter setting where $s$ is generated first corresponds to allowing shared randomness in the two-player communication game; however, this is not relevant to the sequential prediction problem.}

Given the fact that this mutual information is not even an upper bound on the amount
of memory that an optimal algorithm (computationally unbounded, and with
complete knowledge of the distribution) would require,
%on average, to predict optimally,
Proposition~\ref{prop:upperbnd} might be surprising.

\subsection{Implications of Proposition~\ref{prop:upperbnd} and Corollary~\ref{hmm}} 

These results show that a Markov model---a model that cannot capture long-range
dependencies or structure of the data---can predict accurately on \emph{any} data-generating distribution (even those corresponding to complex models such as RNNs), provided the order of the Markov model scales with the complexity of the distribution, as parameterized by the mutual information between the past and future.  Strikingly, this parameterization is indifferent to whether the dependencies in the sequence are relatively short-range as in an HMM that mixes quickly, or very long-range as in an HMM that mixes slowly or does not mix at all.  Independent of the nature of these dependencies, provided the mutual information is small, accurate prediction is possible based only on the most recent few observation.  (See Figure~\ref{fig:cycle_seq} for a concrete illustration of this result in the setting of an HMM  that does not mix and has long-range dependencies.)

	\begin{figure}[h]
		\centering
		\centering
		\includegraphics[height=1.5in]{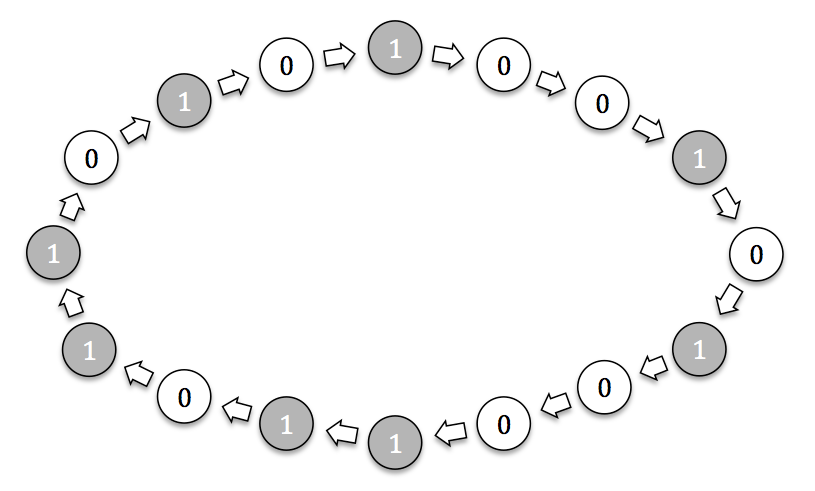}
		\caption{\small{A depiction of a HMM on $n$ states, that repeats a given length $n$ binary sequence of outputs, and hence does not mix.  Corollary~\ref{hmm} and Theorem~\ref{thm:hmm} imply that accurate prediction is possible based  only on short sequences of $O(\log n)$ observations. \label{fig:cycle_seq}}}		
	\end{figure}

At a time when increasingly complex models such as recurrent neural networks
and neural Turing machines are in vogue, these results serve as a baseline theoretical result.  They also help explain the practical success of simple Markov models such as Kneser-Ney smoothing~\cite{kneser1995improved,chen96smoothing} for machine translation and speech recognition systems in the past.   %In the last five years, however, in the midst of the deep learning revolution, Long Short-Term Memory (LSTM) networks have finally demonstrated empirical gains over these Markov-like models \citep{luong2015translation,wu2016google}.   
Although recent recurrent neural networks have yielded empirical gains (see e.g. ~\citep{luong2015translation,wu2016google}), current models still lack the ability to consistently capture long-range dependencies.\footnote{One amusing example is the recent
sci-fi short film \emph{Sunspring} whose script was automatically generated by
an LSTM.  Locally, each sentence of the dialogue
(mostly) makes sense, though there is no cohesion over longer time frames, and no overarching plot trajectory (despite the brilliant acting).}   In some settings, such as natural language, capturing such long-range dependencies seems crucial for achieving human-level results.  Indeed, the main message of a narrative is not conveyed in any single short
segment.   More generally, higher-level intelligence seems to be about the ability to
judiciously decide what aspects of the observation sequence are worth remembering
and updating a model of the world based on these aspects.

Thus, for such settings, Proposition~\ref{prop:upperbnd},  can actually be interpreted as a kind of negative result---that \emph{average} error is not a good metric for training and evaluating models, since models such as the Markov model which are indifferent to the time scale of the dependencies can still perform well under it as long as the number of dependencies is not too large.   It is important to note that average prediction error \emph{is} the metric that ubiquitously used in practice, both in the natural language processing domain and elsewhere. Our results suggest that a different metric might be essential to driving progress towards systems that attempt to capture long-range dependencies and leverage memory in meaningful ways.  We discuss this possibility of alternate prediction metrics more in Section~\ref{futdirections}. 

\medskip
%, this result can be viewed as a damning indictment of using \emph{average prediction error} as the metric of choice when training and evaluating models.  Although this is the metric that pervades NLP practice, our results suggest that a different metric might be essential to developing an effective notion of memory.   
For many other settings, such as financial prediction and lower level language prediction tasks such as those used in OCR, average prediction error \emph{is} actually a meaningful metric.  For these settings, the result of Proposition~\ref{prop:upperbnd} is extremely positive: no matter the nature of the dependencies in the financial markets, it is sufficient to learn a Markov model.  As one obtains more and more data, one can learn a higher and higher order Markov model, and average prediction accuracy should continue to improve.  

For these applications, the question now becomes a computational question: the naive approach to learning an $\ell$-th order Markov model in a domain with an alphabet of size $d$ might require $\Omega(d^{\ell})$ space to store, and data to learn.  \emph{From a computational standpoint, is there a better algorithm?  What properties of the underlying sequence imply that such models can be learned, or approximated more efficiently or with less data?}

Our computational lower bounds, described below, provide some perspective on these computational considerations.

\medskip

%\noindent \textbf{Corollary 1.} \emph{Suppose observations were generated from an underlying HMM
%	$\mathcal{M}$ with $n$ hidden states.
%	Then the $\ell$-th order Markov model obtains $(\log n)/\ell$ average KL error
%	with respect to the optimal predictor that knows the underlying HMM and has access to all past observations.}
%\medskip

%
%\begin{restatable}{corollary}{hmm}[informal] \label{cor:hmm}
%  Suppose observations were generated from an underlying HMM
%  $\mathcal{M}$ with $n$ hidden states.
%  Then the $\ell$-th order Markov model obtains $(\log n)/\ell$ average KL error
%  with respect to the optimal predictor that knows the underlying HMM and has access to the complete history.
%\end{restatable}

%\begin{restatable}{corollary}{hmm}\label{hmm}
%  For any HMM $\mathcal{M}$ with $n$ hidden states and stationary distribution
%  $\pi$, the $\ell$-gram predictions $\mathcal{P}_{\ell}$ get average KL-error,
%  $\delta_{KL}(\mathcal{P}_{\ell}) \le \epsilon$ with respect to the optimal
%  predictor with access to the infinite history for
%  $\ell=\epsilon^{-1}H(\pi)\le \epsilon^{-1}\log n$.
%\end{restatable}

\subsection{Lower bounds}\label{intro:LB}
Our positive results show that accurate prediction is possible via an algorithmically simple model---a Markov model that only depends on the most recent observations---which can be learned in an algorithmically straightforward fashion by simply using the empirical statistics of short sequences of examples, compiled over a sufficient amount of data.  Nevertheless, the Markov model has $d^{\ell}$ parameters, and hence requires an amount of data that scales as $\Omega(d^{\ell})$ to learn, where $d$ is a bound on the size of the observation alphabet.   This prompts the question of whether it is possible to learn a successful predictor based on significantly less data.

We show that, even for the special case where the data sequence is generated from an HMM over $n$ hidden states, this is not possible in general, assuming a natural complexity-theoretic assumption.   An HMM with $n$ hidden states and an output alphabet of size $d$ is defined via only $O(n^2 + nd)$ parameters and $O_{\eps}(n^2+nd)$ samples \emph{are} sufficient, from an information theoretic standpoint, to learn a model that will predict accurately.   While learning an HMM is computationally hard (see e.g.~\cite{mossel2005learning}), this begs the question of whether accurate (average) prediction can be achieved via a computationally efficient algorithm and and an amount of data significantly less than the $d^{\Theta(\log n)}$ that the naive Markov model would require.

%then  be prohibitive to estimate the entire marginals over windows of
%length $\ell = O\left(I(\mathcal{M}) \right)$ required by the naive $\ell$-th
%order Markov process.
Our main lower bound shows that there exists a family of HMMs such that the $d^{\Omega(\log n/\epsilon)}$ sample complexity requirement is necessary for any computationally efficient algorithm that predicts accurately on average, assuming a natural complexity-theoretic assumption. Specifically, we show that this hardness holds, provided that
%we show a computational lower bound
%Any computationally tractable prediction algorithm has
%sample complexity $n^{\Theta(\ell)}$,
 the problem of strongly refuting a certain class of CSPs is hard,
which was conjectured in \citet{feldman2015complexity}
and studied in related works~\citet{allen2015refute} and \citet{kothari2017sum}.
See Section~\ref{sec:lower1} for a description of this class and discussion of the conjectured hardness.

% in ordea scaling of there exist families of sequential models such any efficient model needs $n^{\Omega(\epsilon^{-1}I(\mathcal{M}))}$ samples from the model to make good predictions.
\medskip
\noindent \textbf{Theorem~\ref{high_n}.}
%\begin{theorem}\label{high_n}
\emph{	Assuming the hardness of
	strongly refuting a certain class of CSPs, for all sufficiently large $n$ and any $\eps \in (1/n^c, 0.1)$ for some fixed constant $c$, there exists a family of HMMs with $n$ hidden states and an output alphabet of size $d$ such that any algorithm that runs in time polynomial in $d$, namely time $f(n,\eps)\cdot d^{g(n,\eps)}$ for any functions $f,g$, and achieves average KL or $\ell_1$ error $\eps$ (with respect to the optimal predictor) for a random HMM in the family must observe $d^{\Omega(\log n/\eps)}$
	observations from the HMM.}
%\end{theorem}
\medskip

As the mutual information of the generated sequence of an HMM with $n$ hidden states is bounded by $\log n$, Theorem~\ref{high_n} directly implies that there are families of data-generating distributions $\mathcal{M}$ with mutual information $I(\mathcal{M})$ and observations drawn from an alphabet of size $d$ such that any computationally efficient algorithm requires $d^{\Omega(I(\mathcal{M})/\eps)}$ samples from $\mathcal{M}$ to achieve average error $\epsilon$. The above bound holds when $d$ is large compared to $\log n$ or $I(\mathcal{M})$, but a different but equally relevant regime is where the
alphabet size $d$ is small compared to the scale of dependencies in the sequence
(for example, when predicting characters \citep{kim2015character}).
We show lower bounds in this regime of the same flavor as those of
Theorem~\ref{high_n} except based on the problem of learning a noisy parity
function; the (very slightly) subexponential algorithm of~\citet{blum2003noise}
for this task means that we lose at least a superconstant factor in the
exponent in comparison to the positive results of
Proposition~\ref{prop:upperbnd}.
%based on the presumed hardness of parity
%with noise which suggests that even in this setting, getting a much better
%sample complexity than the naive $\ell$-th order Markov model is probably
%unlikely.\\

%\begin{proposition}\label{binary}
%	Let $f(k)$ denote a lower bound on the amount of time and samples required to learn parity with noise on uniformly random $k$-bit inputs. For all sufficiently large $n$ and $\epsilon\in (1/n^c, 0.1)$ for some fixed constant $c$, there exists a family of HMMs with $n$ hidden states such that any algorithm that achieves average prediction error $\eps$ (with respect to the optimal predictor) for a random HMM in the family requires at least $f\left(\Omega(\log n/\epsilon)\right)$ time or samples.
%\end{proposition}

\medskip
\noindent %\textbf{Proposition~\ref{binary}.} 
 \textbf{Proposition~\hypertarget{prop:bin}{2}.} 
\emph{Let $f(k)$ denote a lower bound on the amount of time and samples required to learn parity with noise on uniformly random $k$-bit inputs. For all sufficiently large $n$ and $\epsilon\in (1/n^c, 0.1)$ for some fixed constant $c$, there exists a family of HMMs with $n$ hidden states such that any algorithm that achieves average prediction error $\eps$ (with respect to the optimal predictor) for a random HMM in the family requires at least $f\left(\Omega(\log n/\epsilon)\right)$ time or samples.}
\medskip

Finally, we also establish the \emph{information theoretic} optimality of the results of Proposition~\ref{prop:upperbnd}, in the sense that among (even computationally unbounded) prediction algorithms that predict based only on the most recent $\ell$ observations, an average KL prediction error of $\Omega(I(\mathcal{M}) / \ell)$ and $\ell_1$ error $\Omega(\sqrt{I(\mathcal{M}) / \ell})$ with respect to the optimal predictor, is necessary.

% to improve on Proposition \ref{prop:upperbnd} beyond constant factors, even for the class of HMMs. This optimality of Proposition \ref{prop:upperbnd} holds both for KL-divergence and for $\ell_1$ distance. Proposition~\ref{prop:upperbnd} via Pinsker's inequality shows that an $\ell_1$ error of $\eps$ can be achieved with windows of length $\ell = I(\mathcal{M}) / \eps^2$. The following proposition shows that even a computationally unbounded prediction algorithm, which has knowledge of the HMM underlying the observations, cannot improve on this $\ell_1$ prediction error beyond a constant factor.

\medskip
\noindent \textbf{Proposition 3. }\emph{There is an absolute constant $c < 1$ such
that for all $0<\epsilon< 1/4$ and sufficiently large $n$, there exists an HMM
with $n$ hidden states such that it is not information-theoretically possible
to obtain average KL prediction error less than $\eps$ or $\ell_1$ error less than $\sqrt{\eps}$ (with respect to the optimal predictor) while using only the most recent $c\log n/\epsilon$ observations to make each prediction.}
\medskip

%\begin{restatable}{theorem}{infobnd}\label{info_bound}
%	For all $0<\epsilon<0.5$ and sufficiently large $n$, there exits an HMM with $n$ states such that it is not information theoretically possible to get average relative zero-one loss or $\ell_1$ loss less than $\epsilon-o(1)$ using windows of length smaller than $c\epsilon^{-2}\log n$, and KL loss less than $\epsilon-o(1)$ using windows of length smaller than $c\epsilon^{-1}\log n$ where $c$ is a fixed constant.
%\end{restatable}
%\paragraph{The Limitations of Being Good on Average.}
%\paragraph{A Better Success Metric?}

\subsection{Future Directions}\label{futdirections}

% ~~~ previous discussion below ~~~

%As mentioned above, for the settings in which capturing long-range dependencies seems essential, it is worth re-examining the choice of metrics that are used to train and evaluate models.  While worst-case prediction error seems too stringent (and hard to implement), there might be a fertile middle ground between average error (which gives too much reward for correctly guessing common words like ``a'' and ``the''), and worst-case error.  One natural possibility would be a sort of re-weighted prediction error that provides more reward for correctly guessing less common words.  [WOULD AN ANALOG OF THE PROPOSITION ALSO HOLD FOR THESE?????]

As mentioned above, for the settings in which capturing long-range dependencies seems essential, it is worth re-examining the choice of ``average prediction error'' as the metric used to train and evaluate models. One possibility, that has a more worst-case flavor, is to only evaluate the algorithm at a chosen set of time steps instead of all time steps.  Hence the naive Markov model can no longer do well just by predicting well on the time steps when prediction is easy. In the context of natural language processing, learning with respect to such a metric intuitively corresponds to training a model to do well with respect to, say, a question answering task instead of a language modeling task.  A fertile middle ground between average error (which gives too much reward for correctly guessing common words like ``a'' and ``the''), and worst-case error might be a re-weighted prediction error that provides more reward for correctly guessing less common observations.  It seems possible, however, that the techniques used to prove Proposition~\ref{prop:upperbnd} can be extended to yield analogous statements for such error metrics.  %The techniques of Proposition \ref{prop:upperbnd} can likely be extended to this setting by measuring mutual information and error with respect to this new reweighed measure that puts more mass on rare words and less mass on frequent words, though this naive extension could drastically increase the mutual information.

In cases where average error is appropriate, given the upper bounds of Proposition~\ref{prop:upperbnd}, it is natural to consider what additional structure might be present that avoids the (conditional) computational lower bounds of Theorem~\ref{high_n}. One possibility is a \emph{robustness} property---for example the property that a Markov model would continue to predict well even when each observation were obscured or corrupted with some small probability. The lower bound instance rely on parity based constructions and hence are very sensitive to noise and corruptions. For learning over \emph{product} distributions, there are well known connections between noise stability and approximation by low-degree polynomials~\cite{o2014analysis,blais2010polynomial}.  Additionally,  low-degree polynomials can be learned agnostically over \emph{arbitrary} distributions via polynomial regression \cite{kalai2008agnostically}.  It is tempting to hope that this thread could be made rigorous, by establishing a connection between natural notions of noise stability over arbitrary distributions, and accurate low-degree polynomial approximations.  Such a connection could lead to significantly better sample complexity requirements for prediction on such ``robust'' distributions of sequences, perhaps requiring only $\poly{d,I(\mathcal{M}),1/\epsilon}$ data.  Additionally, such sample-efficient approaches to learning succinct representations of large Markov models may inform the many practical prediction systems that currently rely on Markov models.

\subsection{Related Work}

%There are two related lines of work, which we discuss:  deep learning models
%(and as it relates to sequential prediction) and universal prediction methods
%from information theory and statistics.

{\flushleft{\textbf{Parameter Estimation. }}}
It is interesting to compare using a Markov model for prediction with methods that attempt to \emph{properly} learn an underlying model.
For example, method of moments algorithms \citep{hsu09spectral,anandkumar12moments}
allow one to estimate a certain class of Hidden Markov model with polynomial
sample and computational complexity.
These ideas have been extended to learning neural networks \citep{sedghi2016training}
and input-output RNNs \citep{janzamin2015beating}.
Using different methods, \citet{arora2014provable} showed how to learn certain
random deep neural networks.
Learning the model directly can result in better sample efficiency,
and also provide insights into the structure of the data.
The major drawback of these approaches is that they usually require the true
data-generating distribution to be in (or extremely close to) the model family that we are learning.  This is a very strong assumption that often does not hold in practice.\\

%A major 
%One bottleneck for some this work is issue of
%identifiability.  Many deep results, yet in many settings involving more
%complex models, we should not expect model identifiability.
%
%Indeed, the point
%of many of these models, particularly deep neural networks, is precisely that
%the model class is extremely expressive and flexible, and we should not expect
%identifiability for such classes for practical parameter regimes. 
{\flushleft{\textbf{Universal Prediction and Information Theory. }}}
On the other end of the spectrum is the class of no-regret online learning methods
which assume that the data generating distribution can even be adversarial
\citep{cesabianchi06prediction}.  However, the nature of these results are
fundamentally different from ours: whereas we are comparing
to the perfect model that can look at the infinite past,
online learning methods typically compare to a fixed set of experts,
which is much weaker. We note that information theoretic tools have also been employed in the online learning literature to show near-optimality of Thompson sampling with respect to a fixed set of experts in the context of online learning with prior information \citep{russo2016information}, Proposition~\ref{prop:upperbnd} can be thought of as an analogous statement about the strong performance of Markov models with respect to the optimal predictions in the context of sequential prediction.
%which in this setting might be ones that predict the same observation.

There is much work on sequential prediction based on KL-error from
the information theory and statistics communities.
The philosophy of these approaches are often more adversarial, with
perspectives ranging from minimum description
length~\cite{bry98,grunwald05} and individual sequence
settings~\cite{dawid84}, where no model of the data distribution process is assumed.
Regarding worst case guarantees (where there is no data
generation process), and \emph{regret} as the notion of optimality, there
is a line of work on both minimax rates and the performance of 
Bayesian algorithms, the latter of which has favorable guarantees in a sequential
setting. Regarding minimax rates, \cite{shtarkov87} provides an
exact characterization of the minimax strategy, though the
applicability of this approach is often limited to settings where the number
strategies available to the learner is relatively small (i.e., the
normalizing constant in \cite{shtarkov87} must exist). More generally,
there has been considerable work on the regret in information-theoretic and
statistical settings, such as the works in
\cite{dawid84,AW01,F91,opper98worst,cesa01worst,Vovk01,KN04,long_version}.

Regarding log-loss more broadly, there is considerable work on
information consistency (convergence in distribution) and minimax
rates with regards to statistical estimation in parametric and non-parametric
families~\cite{Clarke:90,Haussler:97,Barron:98a,Barron:99,Diaconis:86,Zhang:04}.
In some of these settings, e.g. minimax risk in parametric, i.i.d.
settings, there are characterizations of the regret in terms of mutual
information~\cite{Haussler:97}.  

There is also work on universal lossless data compression algorithm,
such as the celebrated Lempel-Ziv algorithm \cite{LempelZiv}.  Here,
the setting is rather different as it is one of coding the entire
sequence (in a block setting) rather than prediction loss.\\

{\flushleft{\textbf{Sequential Prediction in Practice. }}}
Our work was initiated by the
desire to understand the role of memory in sequential prediction, and the
belief that modeling long-range dependencies is important for complex tasks
such as understanding natural language.  There have been many proposed models with explicit notions of memory, including recurrent neural networks~\cite{rumelhart86}, Long Short-Term Memory (LSTM) networks\citep{hochreiter1997lstm,gers2000learning}, attention-based models 	\citep{bahdanau2014neural,vaswani2017attention}, neural Turing machines \cite{graves2014neural}, memory networks \cite{weston2015memory}, differentiable neural computers \cite{graves2016hybrid}, etc.  While some of these models often fail to capture long range dependencies (for example, in the case of LSTMs, it is not difficult to show that they forget the past exponentially quickly if they are ``stable''~\cite{bengio1994learning}), the empirical performance in some settings is quite promising (see, e.g.~\cite{luong2015translation,wu2016google}).  %, they often still fail to capture many long-range dependencies--in the case of LSTMs, for example, it is not difficult to show that they forget the past exponentially quickly if they are ``stable''~\cite{bengio1994learning}.  To gain more insight into this problem, we began by analyzing the simplest Markov predictor, and found to our surprise that it performed nearly as well as one could hope.

%Finally, we already discussed the very related work of~\citet{harsha2007communication}, which considers generating a sample from a joint distribution over random variables $(A,B)$ via the following communication scheme: the sender samples $A$ and transmits some code $f(A)$, while the receiver samples $B$ conditioned on $f(A)$.
%given a specified joint distribution over $(A,B)$.
%This work frames this question from the
%perspective of a communication game in which one player will sample $A$ and the
%other will sample $B$ given some $f(A)$.
%In the setting where the two parties have shared randomness (which is not applicable to our sequential prediction setting),
%the expected number of bits in $f(A)$ lies in $[I(A;B), 2 I(A;B)]$. Without shared randomness, the communication cost can be exponential in $I(X;Y)$.

%\input{summary}
\section{Proof Sketch of Theorem~\ref{thm:hmm}}\label{sec:hmm}

We provide a sketch of the proof of Theorem~\ref{thm:hmm}, which gives stronger guarantees than Proposition \ref{prop:upperbnd} but only applies to sequences generated from an HMM. The core of this proof is the following lemma that guarantees that the Markov model that knows the true marginal probabilities of all short sequences, will end up predicting well.  Additionally, the bound on the expected prediction error will hold in expectation over \emph{only} the randomness of the HMM during the short window, and with high probability over the randomness of when the window begins (our more general results hold in expectation over the randomness of when the window begins).  For settings such as financial forecasting, this additional guarantee is particularly pertinent; you do not need to worry about the possibility of choosing an ``unlucky'' time  to begin your trading regime, as long as you plan to trade for a duration that spans an entire short window.    Beyond the extra strength of this result for HMMs, the proof approach is intuitive and pleasing, in comparison to the more direct information-theoretic proof of Proposition~\ref{prop:upperbnd}.   We first state the lemma and sketch its proof, and then conclude the section by describing how this yields Theorem~\ref{thm:hmm}. %We also note that Lemma~\ref{lem:regret} and Theorem~\ref{thm:hmm} only hold with respect to $\ell_1$ error, whereas Propostion~\ref{prop:upperbnd} also holds with respect to KL loss. We believe this is necessary and that it is not possible to get high probability statements as in Lemma~\ref{lem:regret} for KL loss as KL loss can be unbounded.  \\

%\medskip
%\vspace{-6pt}
%\noindent
%\textbf{Lemma~\ref{lem:regret}.} \emph{
\begin{lemma}\label{lem:regret}
Consider an HMM with $n$ hidden states,  let the hidden state at time $s=0$ be chosen according to an arbitrary distribution $\pi,$ and denote the observation at time $s$ by $x_s$.  Let $OPT_s$ denote the conditional distribution of $x_s$ given observations $x_0,\ldots,x_{s-1}$, and knowledge of the hidden state at time $s=0$.  Let $M_s$ denote the conditional distribution of $x_s$ given only $x_0,\ldots,x_{s-1},$ which corresponds to the naive $s$-th order Markov model that knows only the joint probabilities of sequences of the first $s$ observations. Then with probability at least $1-1/n^{c-1}$ over the choice of initial state, for $\ell=c \log n/\eps^2$, $c\ge 1$ and $\epsilon \ge 1/\log^{0.25} n$,  $$\E\Big[\sum_{s=0}^{\ell-1} \| OPT_s - M_s \|_1 \Big] \le 4\eps \ell,$$ where the expectation is with respect to the randomness  in the outputs $x_0,\ldots,x_{\ell-1}.$
\end{lemma}
%\medskip

%Given Lemma~\ref{hmm_whp}, the proof of Theorem~\ref{thm:hmm} proceeds by assembling the following pieces: 1) for an appropriately chosen constant $c$, with high probability, the empirical distribution of windows of length $\ell$ compiled over a sequence $x_0,\ldots,x_t$ of length $t=d^{c \ell}$ will be close to the distribution corresponding to drawing

The proof of the this lemma will hinge on establishing a connection between $OPT_s$---the Bayes optimal model that knows the HMM and the initial hidden state $h_0$, and at time $s$ predicts the true distribution of $x_s$ given $h_0, x_0,\ldots,x_{s-1}$---and the naive order $s$ Markov model $M_s$ that knows the joint probabilities of sequences of $s$ observations (given that the initial state is drawn according to $\pi$), and predicts accordingly.   This latter model is precisely the same as the model that knows the HMM and distribution $\pi$ (but not $h_0$), and outputs the conditional distribution of $x_s$ given the observations.  

To relate these two models, we proceed via a martingale argument that leverages the intuition that, at each time step either $OPT_s \approx M_s$, or, if they differ significantly, we expect the $s$th observation $x_s$ to contain a significant amount of information about the hidden state at time zero, $h_0$, which will then improve $M_{s+1}$.   Our submartingale will precisely capture the sense that for any $s$ where there is a significant deviation between $OPT_s$ and $M_s$, we expect the probability of the initial state being $h_0$ conditioned on $x_0,\ldots,x_s$, to be significantly more than the probability of $h_0$ conditioned on $x_0,\ldots,x_{s-1}.$  

More formally, let $H^s_0$ denote the distribution of the hidden state at time $0$ conditioned on $x_0,\ldots,x_s$ and let $h_0$ denote the true hidden state at time 0. Let $H^s_0(h_0)$ be the probability of $h_0$ under the distribution $H^s_0$. We show that the following expression is a submartingale: 
$$\log \left( \frac{H^s_0(h_0)}{1-H^s_0(h_0)}\right)-\frac{1}{2}\sum_{i=0}^s \|OPT_i - M_i\|_1^2.$$ 
The fact that this is a submartingale is not difficult: Define $R_{s}$ as the conditional distribution of $x_s$ given observations $x_0,\dotsb,x_{s-1}$ and initial state drawn according to $\pi$ but \emph{not} being at hidden state $h_0$ at time 0. Note that $M_s$ is a convex combination of $OPT_s$ and $R_s$, hence ${\norm{OPT_s-M_s}{1}}\le \;{\norm{OPT_s-R_s}{1}}$. To verify the submartingale property, note that by Bayes Rule, the change in the LHS at any time step $s$ is the log of the ratio of the probability of observing the output $x_s$ according to the distribution $OPT_{s}$ and the probability of $x_s$ according to the distribution $R_{s}$. The expectation of this is the KL-divergence between $OPT_s$ and $R_s$, which can be related to the $\ell_1$ error using Pinsker's inequality.

%[VATSAL---PLEASE EDIT HERE] by Bayes Rule, ratio of num to denominator is ratio of probability of observing any output conditional on being at $s_0$ versus not at $s_0$, the expectation of which is the KL-divergence between the prediction conditioned on being $s_0$ versus not being at $s_0$.

At a high level, the proof will then proceed via concentration bounds (Azuma's inequality), to show that, with high probability, if  the error from the first $\ell= c \log n /\eps^2$ timesteps is large, then $\log\left(\frac{H_0^{\ell-1}( h_0)}{1-H_0^{\ell-1}( h_0)}\right)$ is also likely to be large, in which case the posterior distribution of the hidden state, $H_0^{\ell-1}$ will be sharply peaked at the true hidden state, $h_0$, unless $h_0$ had negligible mass (less than $n^{-c}$) in distribution $\pi.$ 

There are several slight complications to this approach, including the fact that the submartingale we construct does not necessarily have nicely concentrated or bounded differences, as the first term in the submartingale could change arbitrarily.   We address this by noting that the first term should not decrease too much except with tiny probability, as this corresponds to the posterior probability of the true hidden state sharply dropping.  For the other direction, we can simply ``clip'' the deviations to prevent them from exceeding $\log n$ in any timestep, and then   show that the submartingale property continues to hold despite this clipping by proving the following modified version of Pinsker's inequality: 

\begin{restatable}{lemma}{coolKL}\label{cool_pinsker} (Modified Pinsker's inequality)
		For any two distributions $\mu(x)$ and $\nu(x)$ defined on $x \in X$, define the $C$-truncated KL divergence as $\tilde{D}_C(\mu\parallel \nu) = \E_{\mu}\Big[\log \Big(\min\Big\{\frac{\mu(x)}{\nu(x)},C\Big\}\Big)\Big] $ for some fixed $C$ such that $\log C \ge8$. Then $\tilde{D}_C(\mu\parallel \nu) \ge {\frac{1}{2}\norm{\mu-\nu }{1}^2}$.
\end{restatable}

\medskip

Given Lemma~\ref{lem:regret}, the proof of Theorem~\ref{thm:hmm} follows relatively easily.  Recall that Theorem~\ref{thm:hmm} concerns the expected prediction error at a timestep $t \leftarrow \{0,1,\ldots, d^{c \ell}\}$, based on the model $M_{emp}$ corresponding to the empirical distribution of length $\ell$ windows that have occurred in $x_0,\ldots,x_t,$.  The connection between the lemma and theorem is established by showing that, with high probability, $M_{emp}$ is close to $M_{\hat{\pi}},$ where $\hat{\pi}$ denotes the empirical distribution of (unobserved) hidden states $h_0,\ldots,h_t$, and $M_{\hat{\pi}}$ is the distribution corresponding to drawing the hidden state $h_0 \leftarrow \hat{\pi}$ and then generating $x_0,x_1,\ldots,x_{\ell}.$  We provide the full proof in Appendix~\ref{sec:hmm_app}.

\section{Definitions and Notation}

%For a random variable $X$ defined over a support $\mathcal{X}$ we define it's entropy as $H(X)=-\sum_{x \in \mathcal{X}}\Pr(x) \log \Pr(x)$. The conditional entropy $H(Y|X)$ is defined as $H(Y|X)=\sum_{x \in \mathcal{X}}^{}\Pr(x)H(Y|X=x)$. 
Before proving our general Proposition \ref{prop:upperbnd}, we first introduce the necessary notation.  For any random variable $X$, we denote its distribution as $Pr(X)$. The mutual information between two random variables $X$ and $Y$ is defined as $I(X;Y)=H(Y)-H(Y|X)$ where $H(Y)$ is the entropy of $Y$ and $H(Y|X)$ is the conditional entropy of $Y$ given $X$. The conditional mutual information $I(X;Y|Z)$ is defined as:
\begin{align*}
I(X;Y|Z)&=H(X|Z)-H(X|Y,Z) =\E_{x,y,z} \log \frac{\Pr(X|Y,Z)}{\Pr(X|Z)}\\
&=\E_{y,z} D_{KL}(\Pr(X|Y,Z) \parallel \Pr(X|Z) ),
\end{align*}
%\begin{align}
%I(X;Y|Z)&=H(X|Z)-H(X|Y,Z) \nonumber\\
%&=\E_{x,y,z} \log \frac{\Pr(X|Y,Z)}{\Pr(Y|Z)}\nonumber\\
%&=\E_{y,z} D_{KL}(\Pr(X|Y,Z) \parallel \Pr(X|Z) ) \nonumber
%\end{align}
where $D_{KL}(p\parallel q)=\sum_{x}^{}p(x)\log\frac{p(x)}{q(x)}$ is the KL divergence between the distributions $p$ and $q$. Note that we are slightly abusing notation here as $D_{KL}(\Pr(X|Y,Z) \parallel \Pr(X|Z) )$ should technically be ${D_{KL}(\Pr(X|Y=y,Z=z) \parallel \Pr(X|Z=z) )}$. But we will ignore the assignment in the conditioning when it is clear from the context. Mutual information obeys the following chain rule: $I(X_1,X_2;Y)=I(X_1;Y)+ I(X_2;Y|X_1)$.\\

Given a distribution over infinite sequences, $\{x_t\}$ generated by some model $\mathcal{M}$ where $x_t$ is a random variable denoting the output at time $t$, we will use the shorthand $x_{i}^j$ to denote the collection of random variables for the subsequence of outputs $\{x_i, \dotsb, x_j\}$. The distribution of $\{x_t\}$ is \emph{stationary} if the joint distribution of any subset of the sequence of random variables $\{x_t\}$ is invariant with respect to shifts in the time index. Hence $\Pr(x_{i_1}, x_{i_2}, \dotsb, x_{i_n})=\Pr(x_{i_1+l}, x_{i_2+l}, \dotsb, x_{i_n+l})$ for any $l$ if the process is stationary.

We are interested in studying how well the output $x_t$ can be predicted by an algorithm which only looks at the past $\ell$ outputs. The predictor $\mathcal{A}_{\ell}$ maps a sequence of $\ell$ observations to a predicted distribution of the next observation.  We denote the predictive distribution of $\mathcal{A}_{\ell}$ at time $t$ as ${Q}_{\mathcal{A}_{\ell}}(x_{t}|x_{t-\ell}^{t-1})$. We refer to the Bayes optimal predictor using only windows of length $\ell$ as $\mathcal{P}_{\ell}$, hence the prediction of $\mathcal{P}_\ell$ at time $t$ is $\Pr(x_{t}|x_{t-\ell}^{t-1})$. Note that $\mathcal{P}_{\ell}$ is just the naive $\ell$-th order Markov predictor provided with the true distribution of the data. We denote the Bayes optimal predictor that has access to the entire history of the model as $\mathcal{P}_{\infty}$, the prediction of $\mathcal{P}_{\infty}$ at time $t$ is $\Pr(x_t|x_{-\infty}^{t-1})$. We will evaluate average performance of the predictions of $\mathcal{A}_{\ell}$ and $\mathcal{P}_{\ell}$ with respect to $\mathcal{P}_{\infty}$ over a long time window $[0:T-1]$.

The crucial property of the distribution that is relevant to our results is the mutual information between past and future observations. For a stochastic process  $\{x_t\}$ generated by some model $\mathcal{M}$ we define the mutual information $I(\mathcal{M})$ of the model $\mathcal{M}$ as the  mutual information between the past and future, averaged over the window $[0:T-1]$,
\begin{align}
I(\mathcal{M}) = \lim_{T\rightarrow \infty} \frac{1}{T}\sum_{t=0}^{T-1} I(x_{-\infty}^{t-1};x_{t}^{\infty}). \label{mi}
\end{align}
If the process $\{x_t\}$ is stationary, then $I(x_{-\infty}^{t-1};x_{t}^{\infty})$ is the same for all time steps hence $I(\mathcal{M})=I(x_{-\infty}^{-1};x_{0}^{\infty})$.  If the average does not converge and hence the limit in \eqref{mi} does not exist, then we can define $I(\mathcal{M},[0:T-1])$ as the mutual information for the window $[0:T-1]$, and the results hold true with $I(\mathcal{M})$ replaced by ${I(\mathcal{M},[0:T-1])}$.

We now define the metrics we consider to compare the predictions of $\mathcal{P}_{\ell}$ and $\mathcal{A}_{\ell}$ with respect to $\mathcal{P}_{\infty}$. Let $F(P,Q)$ be some measure of distance between two predictive distributions. In this work, we consider the KL-divergence, $\ell_1$ distance and the relative zero-one loss between the two distributions. The KL-divergence and $\ell_1$ distance between two distributions are defined in the standard way. We define the relative zero-one loss as the difference between the zero-one loss of the optimal predictor $\mathcal{P}_{\infty}$ and the algorithm $\mathcal{A}_{\ell}$. We define the expected loss of any predictor $\mathcal{A}_{\ell}$ with respect to the optimal predictor $\mathcal{P}_{\infty}$ and a loss function $F$ as follows:
\begin{align*}
\delta_{F}^{(t)}(\mathcal{A}_{\ell}) &=\E_{x_{-\infty}^{t-1}}\Big[F( \Pr(x_{t}|x_{-\infty}^{t-1}),  {Q}_{\mathcal{A}_{\ell}}(x_{t}|x_{t-\ell}^{t-1}))\Big],\\ % \quad \quad
\delta_{F}(\mathcal{A}_{\ell}) &=\lim_{T\rightarrow \infty}\frac{1}{T}\sum_{t=0}^{T-1}\delta_{F}^{(t)}(\mathcal{A}_{\ell}).\nonumber
\end{align*}
We also define $\hat{\delta}_{F}^{(t)}(\mathcal{A}_{\ell})$ and $\hat{\delta}_{F}(\mathcal{A}_{\ell})$ for the algorithm $\mathcal{A}_{\ell}$ in the same fashion as the error in estimating ${P}(x_{t}|x_{t-\ell}^{t-1})$, the true conditional distribution of the model $\mathcal{M}$.
\begin{align*}
\hat{\delta}_{F}^{(t)}(\mathcal{A}_{\ell}) &=\E_{x_{t-\ell}^{t-1}}\Big[F( \Pr(x_{t}|x_{t-\ell}^{t-1}), {Q}_{\mathcal{A}_{\ell}}(x_{t}|x_{t-\ell}^{t-1}))\Big],\\% \quad \quad
\hat{\delta}_{F}(\mathcal{A}_{\ell}) &=\lim_{T\rightarrow \infty}\frac{1}{T}\sum_{t=0}^{T-1}\hat{\delta}_{F}^{(t)}(\mathcal{A}_{\ell}).\nonumber
\end{align*}

\section{Predicting Well with Short Windows}\label{sec:upper}

 To establish our general proposition, which applies beyond the HMM setting, we provide an elementary and purely information theoretic proof.

\begin{restatable}{proposition}{upperbnd}\label{prop:upperbnd}
	For any  data-generating distribution $\mathcal{M}$ with mutual information $I(\mathcal{M})$ between past and future observations,
  the best $\ell$-th order Markov model $\mathcal{P}_{\ell}$ obtains average KL-error,
  $\delta_{KL}(\mathcal{P}_{\ell}) \le I(\mathcal{M})/\ell$
  with respect to the optimal predictor with access to the infinite history.
  Also, any predictor $\mathcal{A}_{\ell}$ with $\hat{\delta}_{KL}(\mathcal{A}_{\ell})$
  average KL-error in estimating the joint probabilities over windows of length $\ell$ gets average error
  $\delta_{KL}(\mathcal{A}_{\ell}) \le I(\mathcal{M})/\ell + \hat{\delta}_{KL}(\mathcal{A}_{\ell})$.
\end{restatable}
\begin{proof}
	We  bound the expected error by splitting the time interval $0$ to $T-1$ into blocks of length $\ell$. Consider any block starting at time $\tau$. We find the average error of the predictor from time $\tau$ to $\tau+\ell-1$ and then average across all blocks.
	
	To begin, note that we can decompose the error as the sum of the error due to not knowing the past history beyond the most recent $\ell$ observations and the error in estimating the true joint distribution of the data over a $\ell$ length block. Consider any time $t$. Recall the definition of $\delta_{KL}^{(t)}(\mathcal{A}_{\ell})$,
	\begin{align*}
	\delta_{KL}^{(t)}(\mathcal{A}_{\ell})&=\E_{x_{-\infty}^{t-1}}\Big[D_{KL}( \Pr(x_{t}|x_{-\infty}^{t-1})\parallel {Q}_{\mathcal{A}_{\ell}}(x_{t}|x_{t-\ell}^{t-1}))\Big] \nonumber\\
	&=\E_{x_{-\infty}^{t-1}}\Big[D_{KL}( \Pr(x_{t}|x_{-\infty}^{t-1})\parallel \Pr(x_{t}|x_{t-\ell}^{t-1}))\Big] \\
	&\;+ \E_{x_{-\infty}^{t-1}}\Big[D_{KL}( \Pr(x_{t}|x_{t-\ell}^{t-1})\parallel {Q}_{\mathcal{A}_{\ell}}(x_{t}|x_{t-\ell}^{t-1}))\Big] \nonumber\\
	&= \delta_{KL}^{(t)}(\mathcal{P}_{\ell})+\hat{\delta}_{KL}^{(t)}(\mathcal{A}_{\ell}).\nonumber
	\end{align*}
	Therefore, $\delta_{KL}(\mathcal{A}_{\ell}) =\delta_{KL}(\mathcal{P}_{\ell})+\hat{\delta}_{KL}(\mathcal{A}_{\ell})$. 
	It is easy to verify that $\delta_{KL}^{(t)}(\mathcal{P}_{\ell})=I(x_{-\infty}^{t-\ell-1};x_t|x_{t-\ell}^{t-1})$. This relation formalizes the intuition that the current output ($x_t$) has significant extra information about the past ($x_{-\infty}^{t-\ell-1}$) if we cannot predict it as well using the $\ell$ most recent observations ($x_{t-\ell}^{t-1}$), as can be done by using the entire past $(x_{-\infty}^{t-1})$. We will now upper bound the total error for the window $[\tau, \tau+\ell -1]$. We expand $I(x_{-\infty}^{\tau-1}; x_{\tau}^{\infty})$ using the chain rule, %\vspace{-.3cm}
	$$		I(x_{-\infty}^{\tau-1}; x_{\tau}^{\infty}) = \sum_{t=\tau}^{\infty}I(x_{-\infty}^{\tau-1}; x_t|x_{\tau}^{t-1}) \ge \sum_{t=\tau}^{\tau+\ell-1}I(x_{-\infty}^{\tau-1}; x_t|x_{\tau}^{t-1}).$$
	Note that $I(x_{-\infty}^{\tau-1}; x_t|x_{\tau}^{t-1}) \ge I(x_{-\infty}^{t-\ell-1}; x_t|x_{t-\ell}^{t-1}) = \delta_{KL}^{(t)}(\mathcal{P}_{\ell})$ as ${t-\ell} \le \tau$ and $I(X,Y;Z) \ge I(X;Z|Y)$. The proposition now follows from averaging the error across the $\ell$ time steps and using Eq. \ref{mi} to average over all blocks of length $\ell$ in the window $[0,T-1]$, $$\frac{1}{\ell}\sum_{t=\tau}^{\tau+\ell-1} \delta_{KL}^{(t)}(\mathcal{P}_{\ell}) \le\frac{1}{\ell} I(x_{-\infty}^{\tau-1}; x_{\tau}^{\infty}) \implies \delta_{KL}(\mathcal{P}_{\ell}) \le \frac{I(\mathcal{M})}{\ell}.$$ 
\end{proof}

%\vspace{-1cm}
Note that Proposition \ref{prop:upperbnd} also directly gives guarantees for the scenario where the task is to predict the distribution of the next block of outputs instead of just the next immediate output, because KL-divergence obeys the chain rule. 

The following easy corollary, relating  KL error to $\ell_1$ error yields the following statement, which also trivially applies to zero/one loss with respect to that of the optimal predictor, as the expected relative zero/one loss at any time step is at most the $\ell_1$ loss at that time step.
\begin{corollary}\label{l1}
	For any  data-generating distribution $\mathcal{M}$ with mutual information $I(\mathcal{M})$ between past and future observations,
	the best $\ell$-th order Markov model $\mathcal{P}_{\ell}$ obtains average $\ell_1$-error
	$\delta_{\ell_1}(\mathcal{P}_{\ell}) \le \sqrt{I(\mathcal{M})/2\ell}$
	with respect to the optimal predictor that has access to the infinite history.
	Also, any predictor $\mathcal{A}_{\ell}$ with $\hat{\delta}_{\ell_1}(\mathcal{A}_{\ell})$
	average $\ell_1$-error in estimating the joint probabilities gets average prediction error
	$\delta_{\ell_1}(\mathcal{A}_{\ell}) \le \sqrt{I(\mathcal{M})/2\ell} + \hat{\delta}_{\ell_1}(\mathcal{A}_{\ell})$.
\end{corollary}
\begin{proof}
	We again decompose the error as the sum of the error in estimating $\hat{P}$ and the error due to not knowing the past history using the triangle inequality.
	\begin{align*}
	\delta_{\ell_1}^{(t)}(\mathcal{A}_{\ell})&=\E_{x_{-\infty}^{t-1}}\Big[\norm{\Pr(x_{t}|x_{-\infty}^{t-1})- {Q}_{\mathcal{A}_{\ell}}(x_{t}|x_{t-\ell}^{t-1})}{1}\Big] \nonumber\\
	&\le\E_{x_{-\infty}^{t-1}}\Big[\norm{\Pr(x_{t}|x_{-\infty}^{t-1})- \Pr(x_{t}|x_{t-\ell}^{t-1})}{1}\Big]\\
	&\;+\E_{x_{-\infty}^{t-1}}\Big[\norm{\Pr(x_{t}|x_{t-\ell}^{t-1})- {Q}_{\mathcal{A}_{\ell}}(x_{t}|x_{t-\ell}^{t-1})}{1}\Big]\nonumber\\
	&= {\delta}_{\ell_1}^{(t)}(\mathcal{P}_{\ell}) +\hat{\delta}_{\ell_1}^{(t)}(\mathcal{A}_{\ell})\nonumber
	\end{align*}
Therefore, $\delta_{\ell_1}(\mathcal{A}_{\ell}) \le\delta_{\ell_1}(\mathcal{P}_{\ell})+\hat{\delta}_{\ell_1}(\mathcal{A}_{\ell})$. By Pinsker's inequality and Jensen's inequality, $\delta_{\ell_1}^{(t)}(\mathcal{A}_{\ell})^2 \le \delta_{KL}^{(t)}(\mathcal{A}_{\ell}) /2$. Using Proposition \ref{prop:upperbnd},
\begin{align}
\delta_{KL}(\mathcal{A}_{\ell})&=\frac{1}{T}\sum_{t=0}^{T-1}\delta_{KL}^{(t)}(\mathcal{A}_{\ell}) \le \frac{I(\mathcal{M})}{\ell}\nonumber
\end{align}
Therefore, using Jensen's inequality again, $\delta_{\ell_1}(\mathcal{A}_{\ell})\le \sqrt{I(\mathcal{M})/2\ell}$.
\end{proof}

\section{Lower Bound for Large Alphabets}\label{sec:lower1}

%The first step to proving tight lower bounds on the sample complexity in the higher alphabet case is to not have output distribution at any time step $x_t$ have support over a large alphabet and be predictable from previous time steps $x_{-\infty}^{t-1}$, as this would increase the mutual information of the model. Hence, we would like to find a model with a large output alphabet, but with the predictable time steps having a small

 %but with an encoding from the large alphabet to a smaller alphabet (of say size 2), so that the mutual information between the past and the future is kept small. Once framed this way, the problem can now be seen as that of recovering some planted encoding, which leads us to formulating the sequential prediction task as a CSP. 
 Our lower bounds for the sample complexity in the large alphabet case leverage a class of Constraint Satisfaction Problems (CSPs) with high \emph{complexity}. A class of (Boolean) $k$-CSPs is defined via a predicate---a function $P: \{0,1\}^k \rightarrow \{0,1\}.$  An instance of such a $k$-CSP on $n$ variables $\{x_1,\dotsb,x_n\}$ is a collection of sets (clauses) of size $k$ whose $k$ elements consist of $k$ variables or their negations.   Such an instance is \emph{satisfiable} if there exists an assignment to the variables $x_1,\ldots,x_n$ such that the predicate $P$ evaluates to $1$ for every clause.   More generally, the \emph{value} of an instance is the maximum, over all $2^n$ assignments, of the ratio of number of satisfied clauses to the total number of clauses.  
 
 Our lower bounds are based on the presumed hardness of distinguishing \emph{random} instances of a certain class of CSP,  versus instances of the CSP with \emph{high value}.   There has been much work attempting to characterize the difficulty of CSPs---one notion which we will leverage is the \emph{complexity} of a class of CSPs, first defined in~\citet{feldman2015complexity} and studied in~\citet{allen2015refute} and \citet{kothari2017sum}:
 
 \begin{definition}
 	The \emph{complexity} of a class of $k$-CSPs defined by predicate $P: \{0,1\}^k \rightarrow \{0,1\}$ is the largest $r$ such that there exists a distribution supported on the support of $P$ that is $(r-1)$-wise independent (i.e. ``uniform''), and no such $r$-wise independent distribution exists.
 \end{definition}
 
 \begin{example}
 	Both $k$-XOR and $k$-SAT are well-studied classes of $k$-CSPs, corresponding, respectively, to the predicates $P_{XOR}$ that is the XOR of the $k$ Boolean inputs, and $P_{SAT}$ that is the OR of the inputs. These predicates both support $(k-1)$-wise uniform distributions, but not $k$-wise uniform distributions, hence their complexity is $k$.  In the case of $k$-XOR, the uniform distribution over $\{0,1\}^k$ restricted to the support of $P_{XOR}$ is $(k-1)$-wise uniform. The same distribution is also supported by $k$-SAT.
 \end{example}
 
% The \emph{complexity} of a CSP is the largest $r$ such that the CSP supports a $(r-1)$-wise uniform distribution on its predicates but not a $r$-wise uniform distribution.
A random instance of a CSP with predicate $P$ is an instance such that all the clauses are chosen uniformly at random (by selecting the $k$ variables uniformly, and independently negating each variable with probability $1/2$). A random instance will have value close to $\E[P]$, where $\E[P]$ is the expectation of $P$ under the uniform distribution. In contrast, a planted instance is generated by first fixing a satisfying assignment $\boldsymbol{\sigma}$ and then sampling clauses that are satisfied, by uniformly choosing $k$ variables, and picking their negations according to a $(r-1)$-wise independent distribution associated with the predicate. Hence a planted instance always has value 1. A noisy planted instance with planted assignment $\boldsymbol{\sigma}$ and noise level $\eta$ is generated by sampling consistent clauses (as above) with probability $1-\eta$ and random clauses with probability $\eta$, hence with high probability it has value close to $1-\eta+\eta\E[P]$. Our hardness results are based on distinguishing whether a CSP instance is random versus has a high value (value close to $1-\eta+\eta\E[P]$). 

As one would expect, the difficulty of distinguishing random instances from noisy planted instances, decreases as the number of sampled clauses grows.  The following conjecture of \citet{feldman2015complexity} asserts a sharp boundary on the number of clauses, below which this problem becomes computationally intractable, while remaining information theoretically easy.  \\%The notation is made more explicit in Appendix \ref{sec:largen_app}. \\
 %A similar route was taken by \citet{daniely2014complexity} and \citet{daniely2015complexity} to prove hardness of learning DNFs and halfspaces.
 
	\noindent \textbf{Conjectured CSP Hardness [Conjecture 1]} \cite{feldman2015complexity}:\emph{ Let $Q$ be any distribution over $k$-clauses and $n$ variables of complexity $r$ and $0<\eta<1$. Any polynomial-time (randomized) algorithm that, given access to a distribution $D$ that equals either the uniform distribution over $k$-clauses $U_k$ or a (noisy) planted distribution $Q_{\boldsymbol{\sigma}}^{\eta}=(1-\eta)Q_{\boldsymbol{\sigma}}+\eta U_k$ for some $\boldsymbol{\sigma} \in \{0,1\}^n$ and planted distribution $Q_{\boldsymbol{\sigma}}$, decides correctly whether $D=Q_{{\boldsymbol{\sigma}}}^{\eta}$ or $D=U_k$ with probability at least 2/3 needs $\tilde{\Omega}(n^{r/2})$ clauses. }\\

\citet{feldman2015complexity} proved the conjecture for the class of \emph{statistical algorithms}.\footnote{Statistical algorithms are an extension of the statistical query model.  These are algorithms that do not directly access samples from the distribution but instead have access to estimates of the expectation of any bounded function of a sample, through a ``statistical oracle''. \citet{feldman2013statistical} point out that almost all algorithms that work on random data also work with this limited access to samples, refer to \citet{feldman2013statistical} for more details and examples.} Recently, \citet{kothari2017sum} showed that the natural Sum-of-Squares (SOS) approach requires $\tilde{\Omega}(n^{r/2})$ clauses to refute random instances of a CSP with complexity $r$, hence proving {Conjecture 1} for any polynomial-size semidefinite programming relaxation for refutation. Note that $\tilde{\Omega}(n^{r/2})$ is tight, as \citet{allen2015refute} give a SOS algorithm for refuting random CSPs beyond this regime. Other recent papers such as \citet{daniely2014complexity} and \citet{daniely2015complexity} have also used presumed hardness of strongly refuting random $k$-SAT and random $k$-XOR instances with a small number of clauses to derive conditional hardness for various learning problems.\\ %\citet{mori2016lower} made a similar conjecture (Conjecture 2 of their paper) which proposes that random CSPs cannot be refuted with fewer than $\tilde{\Omega}(n^{r k/2})$ clauses with any polynomial sized SDP relaxation. \citet{odonnell2014goldreich} and \citet{mori2016lower} showed that the Sherali-Adams(SA)\footnote{Polynomial sized SA relaxations are as powerful as any polynomial sized LP relaxation (see \citet{chan2013approximate})}  SDP hierarchy cannot refute a random CSP below the $\tilde{\Omega}(n^{r k/2})$ threshold. We also refer the reader to the work of \citet{barak2015sum} which shows that the stronger Sum-of-Squares hierarchy needs $O(n)$ clauses when the predicate is pairwise uniform and previous work of \citet{benabbas2012sdp} and \citet{tulsiani2013ls+} on hardness of refuting CSPs with a pairwise uniform predicate. \\

A first attempt to encode a $k$-CSP as a sequential model is to construct a model which outputs $k$ randomly chosen literals for the first $k$ time steps 0 to $k-1$, and then their (noisy) predicate value for the final time step $k$. Clauses from the CSP correspond to samples from the model, and the algorithm would need to solve the CSP to predict the final time step $k$. However, as all the outputs up to the final time step are random, the trivial prediction algorithm that guesses randomly and does not try to predict the output at time $k$, would be near optimal. To get strong lower bounds, we will output $m>1$ functions of the $k$ literals after $k$ time steps, while still ensuring that all the functions remain collectively hard to invert without a large number of samples. 

We use elementary results from the theory of error correcting codes to achieve this, and prove hardness due to a reduction from a specific family of CSPs to which Conjecture 1 applies. By choosing $k$ and $m$ carefully, we obtain the near-optimal  dependence on the mutual information and error $\epsilon$---matching the upper bounds implied by Proposition~\ref{prop:upperbnd}. We provide a short outline of the argument, followed by the detailed proof in the appendix.

\subsection{Sketch of Lower Bound Construction}\label{subsec:sketch}

We construct a sequential model $\mathcal{M}$ such that making good predictions on the model requires distinguishing random instances of a $k$-CSP $\mathcal{C}$ on $n$ variables from instances of $\mathcal{C}$ with a high value. The output alphabet of $\mathcal{M}$ is $\{a_i\}$ of size $2n$. We choose a mapping from the $2n$ characters $\{a_i\}$ to the $n$ variables $\{x_i\}$ and their $n$ negations $\{\bar{x}_i\}$. %For brevity, we will refer to the assignment of the variable corresponding to a character under the planted assignment $\boldsymbol{\sigma}$ as the assignment of the character itself under $\boldsymbol{\sigma}$. 
For any clause $C$ and planted assignment $\boldsymbol{\sigma}$ to the CSP $\mathcal{C}$, let $\boldsymbol{\sigma}(C)$ be the $k$-bit string of values assigned by $\boldsymbol{\sigma}$ to literals in $C$. The model $\mathcal{M}$ will output $k$ characters from time $0$ to $k-1$ chosen uniformly at random, which correspond to literals in the CSP $\mathcal{C}$; hence the $k$ outputs correspond to a clause $C$ of the CSP. For some $m$ (to be specified later) we will construct a binary matrix $\mathbf{A}\in\{0,1\}^{m\times k},$ which will correspond to a good error-correcting code. For the time steps $k$ to $k+m-1$, with probability $1-\eta$ the model outputs $\mathbf{y}\in\{0,1\}^m$ where $\mathbf{y}=\mathbf{Av} \mod 2$ and $\mathbf{v}=\boldsymbol{\sigma}(C)$ with $C$ being the clause associated with the outputs of the first $k$ time steps. With the remaining probability, $\eta$, the model outputs $m$ uniformly random bits. Note that the mutual information $I(\mathcal{M})$ is at most $m$ as only the outputs from time $k$ to $k+m-1$ can be predicted. \\

We claim that $\mathcal{M}$ can be simulated by an HMM with $2^{m}(2k+m)+m$ hidden states. This can be done as follows. For every time step from $0$ to $k-1$ there will be $2^{m+1}$ hidden states, for a total of $k2^{m+1}$ hidden states. Each of these hidden states has two labels: the current value of the $m$ bits of $\mathbf{y}$, and an ``output label'' of 0 or 1 corresponding to the output at that time step having an assignment of 0 or 1 under the planted assignment $\boldsymbol{\sigma}$. The output distribution for each of these hidden states is either of the following: if the state has an ``output label'' 0 then it is uniform over all the characters which have an assignment of 0 under the planted assignment $\boldsymbol{\sigma}$, similarly if the state has an ``output label'' 1 then it is uniform over all the characters which have an assignment of 1 under the planted assignment $\boldsymbol{\sigma}$. Note that the transition matrix for the first $k$ time steps simply connects a state $h_1$ at the $(i-1)$th time step to a state $h_2$ at the $i$th time step if the value of $\mathbf{y}$ corresponding to $h_1$ should be updated to the value of $\mathbf{y}$ corresponding to $h_2$ if the output at the $i$th time step corresponds to the ``output label'' of $h_2$. For the time steps $k$ through $(k+m-1)$, there are $2^m$ hidden states for each time step, each corresponding to a particular choice of $\mathbf{y}$. The output of an hidden state corresponding to the $(k+i)$th time step with a particular label $\mathbf{y}$ is simply the $i$th bit of $\mathbf{y}$. Finally, we need an additional $m$ hidden states to output $m$ uniform random bits from time $k$ to $(k+m-1)$ with probability $\eta$. This accounts for a total of $k2^{m+1}+m2^m+m$ hidden states. After $k+m$ time steps the HMM transitions back to one of the starting states at time 0 and repeats. Note that the larger $m$ is with respect to $k$, the higher the cost (in terms of average prediction error) of failing to correctly predict the outputs from time $k$ to $(k+m-1)$. Tuning $k$ and $m$ allows us to control the number of hidden states and average error incurred by a computationally constrained predictor.\\

%For every time step $i$ from $0$ to $k-1$, we maintain $2^m$ hidden states corresponding to $\mathbf{v}_i=0$ and $2^m$ hidden states corresponding to $\mathbf{v}_i=1$. Each of these $2^m$ states stores the current value of the $m$ bits of $\mathbf{y}$. This takes a total of $k2^{m+1}$ hidden states. We use $2^m$ hidden states for each time step $k$ through $k+m-1$ for the $k$ output bits. Finally, we need an additional $m$ hidden states to output $m$ uniform random bits from time $k$ to $k+m-1$ with probability $\eta$. This accounts for a total of $k2^{m+1}+2^m+m$ hidden states. Note that the larger $m$ is with respect to $k$, the higher is the cost (in terms of average prediction error) of failing to correctly predict the outputs from time $k$ to $k+m-1$. Tuning $k$ and $m$ allows us to control the number of hidden states or the mutual information, and average error incurred by a computationally constrained predictor.\\

We define the CSP $\mathcal{C}$ in terms of a collection of predicates $P(\textbf{y})$ for each $\textbf{y}\in \{0,1\}^m$. While Conjecture 1 does not directly apply to $\mathcal{C}$, as it is defined by a collection of predicates instead of a single one, we will later show a reduction from a related CSP $\mathcal{C}_0$ defined by a single predicate for which Conjecture 1 holds. For each $\textbf{y}$, the predicate $P(\textbf{y})$ of $\mathcal{C}$ is the set of $\mathbf{v}\in\{0,1\}^k$ which satisfy ${\mathbf{y}=\mathbf{Av} \mod 2}$. Hence each clause has an additional label $\mathbf{y}$ which determines the satisfying assignments, and this label is just the output of our sequential model $\mathcal{M}$ from time $k$ to $k+m-1$. Hence for any planted assignment $\boldsymbol{\sigma}$, the set of satisfying clauses $C$ of the CSP $\mathcal{C}$ are all clauses such that $\mathbf{Av} = \mathbf{y} \mod 2$ where $\textbf{y}$ is the label of the clause and $\mathbf{v}=\boldsymbol{\sigma}(C)$. We define a (noisy) planted distribution over clauses $Q_{\boldsymbol{\sigma}}^\eta$ by first uniformly randomly sampling a label $\mathbf{y}$, and then sampling a consistent clause with probability $(1-\eta)$, otherwise with probability $\eta$ we sample a uniformly random clause. Let $U_k$ be the uniform distribution over all $k$-clauses with uniformly chosen labels $\mathbf{y}$. We will show that Conjecture 1 implies that distinguishing between the distributions $Q_{\boldsymbol{\sigma}}^\eta$ and $U_k$ is hard without sufficiently many clauses. This gives us the hardness results we desire for our sequential model $\mathcal{M}$: if an algorithm obtains low prediction error on the outputs from time $k$ through $(k+m-1)$, then it can be used to distinguish between instances of the CSP $\mathcal{C}$ with a high value and random instances, as no algorithm obtains low prediction error on random instances. Hence hardness of strongly refuting the CSP $\mathcal{C}$ implies hardness of making good predictions on $\mathcal{M}$.\\

We now sketch the argument for why Conjecture 1 implies the hardness of strongly refuting the CSP $\mathcal{C}$. We define another CSP $\mathcal{C}_0$ which we show reduces to $\mathcal{C}$. The predicate $P$ of the CSP $\mathcal{C}_0$ is the set of all $\mathbf{v}\in \{0,1\}^k$ such that $\mathbf{Av}=0 \mod 2$. Hence for any planted assignment $\boldsymbol{\sigma}$, the set of satisfying clauses of the CSP $\mathcal{C}_0$ are all clauses such that $\mathbf{v}=\boldsymbol{\sigma}(C)$ is in the nullspace of $\mathbf{A}$. As before, the planted distribution over clauses is uniform on all satisfying clauses with probability $(1-\eta)$, with probability $\eta$ we add a uniformly random $k$-clause. For some $\gamma\ge 1/10$, if we can construct $\mathbf{A}$ such that the set of satisfying assignments $\mathbf{v}$ (which are the vectors in the nullspace of $\mathbf{A}$) supports a $(\gamma k-1)$-wise uniform distribution, then by Conjecture 1 any polynomial time algorithm cannot distinguish between the planted distribution and uniformly randomly chosen clauses with less than $\tilde{\Omega}(n^{\gamma k/2})$ clauses. We show that choosing a matrix $\textbf{A}$ whose null space is $(\gamma k-1)$-wise uniform corresponds to finding a binary linear code with rate at least 1/2 and relative distance $\gamma$, the existence of which is guaranteed by the Gilbert-Varshamov bound.\\

We next sketch the reduction from $\mathcal{C}_0$ to $\mathcal{C}$. The key idea is that the CSPs $\mathcal{C}_0$ and $\mathcal{C}$ are defined by linear equations. If a clause $C=(x_1, x_2, \dotsb, x_k)$ in $\mathcal{C}_0$ is satisfied with some assignment $\mathbf{t}\in \{0,1\}^k$ to the variables in the clause then $\mathbf{At}=0 \mod 2$. Therefore, for some $\mathbf{w}\in \{0,1\}^k$ such that $\mathbf{Aw}=\mathbf{y} \mod 2$, $\mathbf{t}+\mathbf{w} \mod 2$ satisfies $\mathbf{A(t+w)}=\mathbf{y} \mod 2$. A clause $C'=(x'_1, x'_2, \dotsb, x'_k)$ with assignment $\mathbf{t+w} \mod 2$ to the variables can be obtained from the clause $C$ by switching the literal $x_i'= \bar{x}_i$ if $\textbf{w}_i=1$ and retaining $x_i'= x_i$ if $\textbf{w}_i=0$. Hence for any label $\textbf{y}$, we can efficiently convert a clause $C$ in $\mathcal{C}_0$ to a clause $C'$ in $\mathcal{C}$ which has the desired label $\textbf{y}$ and is only satisfied with a particular assignment to the variables if $C$ in $\mathcal{C}_0$ is satisfied with the same assignment to the variables. It is also not hard to ensure that we uniformly sample the consistent clause $C'$ in $\mathcal{C}$ if the original clause $C$ was a uniformly sampled consistent clause in $\mathcal{C}_0$.\\

We provide a small example to illustrate the sequential model constructed above. Let $k=3$, $m=1$ and $n=3$. Let $\textbf{A}\in\{0,1\}^{1\times 3}$. The output alphabet of the model $\mathcal{M}$ is $\{a_i,1\le i \le 6\}$. The letter $a_1$ maps to the variable $x_1$, $a_2$ maps to $\bar{x}_1$, similarly $a_3\rightarrow x_2, a_4 \rightarrow \bar{x}_2, a_5 \rightarrow x_3, a_6 \rightarrow \bar{x}_3$. Let $\boldsymbol{\sigma}$ be some planted assignment to $\{x_1,x_2,x_3\}$, which defines a particular model $\mathcal{M}$. %Any assignment of the variables $\{x_1,x_2,x_3\}$ determines a subset $\mathcal{S}$ of $\{a_i\}$ of size 3, where a letter is included in the subset if the corresponding variable is 1. Each choice of subset $\mathcal{S}$ corresponds to a model $\mathcal{M}$, let $\boldsymbol{\sigma}$ be some planted assignment to $\{x_1,x_2,x_3\}$ which defines a subset $\mathcal{S}$ and hence a particular model $\mathcal{M}$. 
If the output of the model $\mathcal{M}$ is $a_1, a_3, a_6$ for the first three time steps, then this corresponds to the clause with literals, $(x_1, x_2, \bar{x}_3)$. For the final time step, with probability $(1-\eta)$ the model outputs $y=\textbf{Av}\mod 2$, with $\mathbf{v}=\boldsymbol{\sigma}(C)$ for the clause $C=(x_1, x_2, \bar{x}_3)$ and planted assignment $\boldsymbol{\sigma}$, and with probability $\eta$ it outputs a uniform random bit. For an algorithm to make a good prediction at the final time step, it needs to be able to distinguish if the output at the final time step is always a random bit or if it is dependent on the clause, hence it needs to distinguish random instances of the CSP from planted instances.

We re-state Theorem \ref{high_n} below in terms of the notation defined in this section, deferring its full proof to Appendix \ref{sec:largen_app}.

\begin{restatable}{theorem}{highn}\label{high_n}
	Assuming Conjecture 1, for all sufficiently large $T$ and
	$1/T^c<\epsilon \le 0.1$ for some fixed constant $c$, there exists a family of HMMs
	with $T$ hidden states and an output alphabet of size $n$ such that, any prediction algorithm that achieves average KL-error, $\ell_1$ error or relative zero-one error less than $\eps$ with probability greater than 2/3  for a randomly chosen HMM in the family, and runs in time $f(T,\eps)\cdot n^{g(T,\eps)}$ for any functions $f$ and $g$,  requires $n^{\Omega(\log T/\eps)}$ samples from the HMM.
\end{restatable}

\section{Lower Bound for Small Alphabets}

Our lower bounds for the sample complexity in the binary alphabet case are based on the average case hardness of the decision version of the parity with noise problem, and the reduction is straightforward.  \\

In the parity with noise problem on $n$ bit inputs we are given examples $\textbf{v}\in \{0,1\}^n$ drawn uniformly from $\{0,1\}^n$ along with their noisy labels $\langle \textbf{s}, \textbf{v}\rangle + \epsilon \mod 2$ where $\textbf{s}\in \{0,1\}^n$ is the (unknown) support of the parity function, and $\epsilon \in\{0,1\}$ is the classification noise such that $\Pr[\epsilon=1]=\eta$ where $\eta<0.05$ is the noise level.\\

Let $Q_{\textbf{s}}^{\eta}$ be the distribution over examples of the parity with noise instance with $\textbf{s}$ as the support of the parity function and $\eta$ as the noise level. Let $U_n$ be the distribution over examples and labels where each label is chosen uniformly from $\{0,1\}$ independent of the example. The strength of of our lower bounds depends on the level of hardness of parity with noise. Currently, the fastest algorithm for the problem due to \citet{blum2003noise} runs in time and samples $2^{n/\log n}$. We define the function $f(n)$ as follows--

\begin{definition}\label{lpn}
	Define $f(n)$ to be the function such that for a uniformly random support $\textbf{s}\in \{0,1\}^n$, with  probability at least $(1-1/n^2)$ over the choice of $\textbf{s}$, any (randomized) algorithm that can distinguish between $Q_{\textbf{s}}^{\eta}$ and $U_n$ with success probability greater than $2/3$ over the randomness of the examples and the algorithm, requires $f(n)$ time or samples.
\end{definition}

Our model will be the natural sequential version of the parity with noise problem, where each example is coupled with several parity bits.  We denote the model as $\mathcal{M}(\textbf{A}_{m \times n})$ for some $\textbf{A} \in \{0,1\}^{m \times n}, m\le n/2$. From time $0$ through $(n-1)$ the outputs of the model are i.i.d. and uniform on $\{0,1\}$. Let $\textbf{v}\in \{0,1\}^n$ be the vector of outputs from time $0$ to $(n-1)$. The outputs for the next $m$ time steps are given by $\textbf{y}=\textbf{Av}+\boldsymbol{\epsilon} \mod 2$,  where $\boldsymbol{\epsilon} \in \{0,1\}^m$ is the random noise and each entry $\epsilon_i$ of $\boldsymbol{\epsilon}$ is an i.i.d random variable such that $\Pr[\epsilon_i=1]=\eta$, where $\eta$ is the noise level. Note that if $\textbf{A}$ is full row-rank, and $\textbf{v}$ is chosen uniformly at random from $\{0,1\}^n$, the distribution of $\textbf{y}$ is uniform on $\{0,1\}^m$. Also $I(\mathcal{M}(\mathbf{A}))\le m$ as at most the binary bits from time $n$ to $n+m-1$ can be predicted using the past inputs. As for the large alphabet case, $\mathcal{M}(\textbf{A}_{m \times n})$ can be simulated by an HMM with $2^{m}(2n+m)+m$ hidden states (see Section \ref{subsec:sketch}). \\

%We define a set of $\mathbf{A}$ matrices, which specifies a family of sequential models. Let $\mathcal{S}$ be the set of all $(m \times n)$ matrices $\textbf{A}$ such that the sub-matrix of $\textbf{A}$ corresponding to all rows but only the first $2n/3$ columns is full row rank. We need this restriction to lower bound $I(\mathcal{M}(\mathbf{A}))$, as otherwise there could be small or no dependence of the parity bits on the inputs from time 0 to $2n/3-1$. We denote $\mathcal{R}$ as the family of models $\mathcal{M}(\textbf{A})$ for $\textbf{A} \in \mathcal{S}$. Lemma \ref{lem:binary_hardness} shows that with high probability over the choice of $\mathbf{A}$, distinguishing outputs from the model $\mathcal{M}(\textbf{A})$ from random examples $U_n$ requires $f(n)$ time or examples. 

We define a set of $\mathbf{A}$ matrices, which specifies a family of sequential models. Let $\mathcal{S}$ be the set of all $(m \times n)$ matrices $\textbf{A}$ such that the $\textbf{A}$ is full row rank. We need this restriction as otherwise the bits of the output $\mathbf{y}$ will be dependent. We denote $\mathcal{R}$ as the family of models $\mathcal{M}(\textbf{A})$ for $\textbf{A} \in \mathcal{S}$. Lemma \ref{lem:binary_hardness} shows that with high probability over the choice of $\mathbf{A}$, distinguishing outputs from the model $\mathcal{M}(\textbf{A})$ from random examples $U_n$ requires $f(n)$ time or examples. 

\begin{restatable}{lemma}{parityA}\label{lem:binary_hardness}
	Let $\textbf{A}$ be chosen uniformly at random from the set $\mathcal{S}$. Then, with probability at least $(1-1/n)$ over the choice $\textbf{A}\in \mathcal{S}$, any (randomized) algorithm that can distinguish the outputs from the model $\mathcal{M}(\textbf{A})$ from the distribution over random examples $U_n$ with success probability greater than $2/3$ over the randomness of the examples and the algorithm needs ${f(n)}$ time or examples.
\end{restatable}

The proof of Proposition \hyperlink{prop:bin}{2} follows from Lemma \ref{lem:binary_hardness} and is similar to the proof for the large alphabet case.

%\begin{restatable}{proposition}{binary}\label{binary}
%	With $f(T)$ as defined in Definition \ref{lpn}, for all sufficiently large $T$ and	$1/T^c<\epsilon \le 0.1$ for some fixed constant $c$, there exists a family of HMMs with $T$ hidden states such that any algorithm that achieves average relative zero-one loss, average $\ell_1$ loss, or average KL loss less than $\epsilon$ with probability greater than 2/3 for a randomly chosen HMM in the family needs, requires ${f}(\Omega(\log T/\epsilon))$ time or samples samples from the HMM.
%\end{restatable}
%\subsection{Information theoretic lower bound for window length for $\ell_1$-error}
\section{Information Theoretic Lower Bounds}

We show that \emph{information theoretically}, windows of length $cI(\mathcal{M})/\epsilon^2$ are necessary to get expected relative zero-one loss less than $\epsilon$. As the expected relative zero-one loss is at most the $\ell_1$ loss, which can be bounded by the square of the KL-divergence, this automatically implies that our window length requirement is also tight for $\ell_1$ loss and KL loss. In fact, it's very easy to show the tightness for the KL loss: choose the simple model which emits uniform random bits from time 0 to $n-1$ and repeats the bits from time 0 to $m-1$ for time $n$ through $n+m-1$. One can then choose $n,m$ to get the desired error $\epsilon$ and mutual information $I(\mathcal{M})$. To get a lower bound for the zero-one loss we use the probabilistic method to argue that there exists an HMM such that long windows are required to perform optimally with respect to the zero-one loss for that HMM. We now state the lower bound and sketch the proof idea.

\medskip
\noindent \textbf{Proposition 3. }\emph{
	There is an absolute constant $c$ such that for all $0<\epsilon<1/4$ and sufficiently large $n$, there exists an HMM with $n$ states such that it is not information theoretically possible to get average relative zero-one loss or $\ell_1$ loss less than $\epsilon$ using windows of length smaller than $c\log n/\epsilon^2$, and KL loss less than $\epsilon$ using windows of length smaller than $c\log n/\epsilon$.}
\medskip

We illustrate the construction in Fig. \ref{hmm_fig1} and provide the high-level proof idea with respect to Fig. \ref{hmm_fig1} below.

	\begin{figure}[ht]
		\centering
		\includegraphics[width=2.5 in]{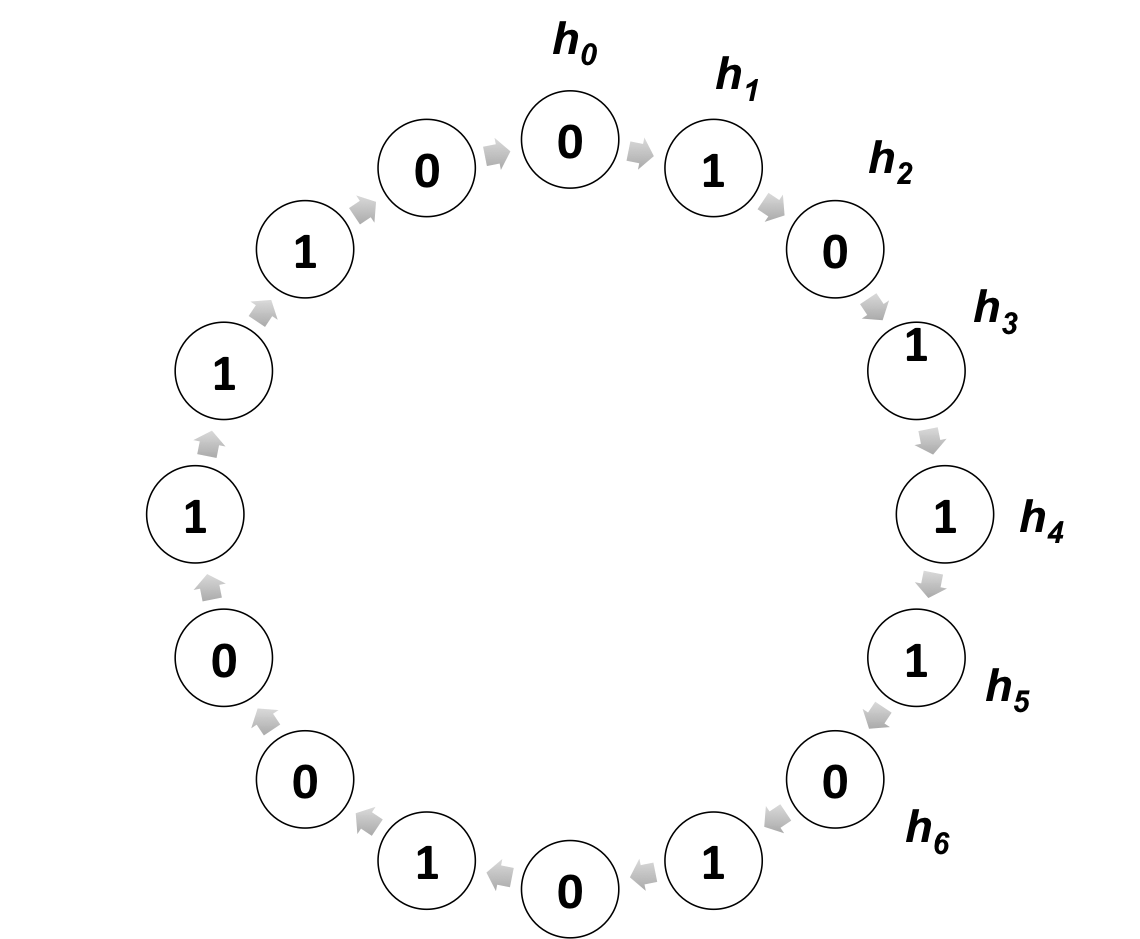}
		\caption{Lower bound construction, $n=16$.}
		\label{hmm_fig1}
	\end{figure}
	
We want show that no predictor $\mathcal{P}$ using windows of length $\ell=3$ can make a good prediction. The transition matrix of the HMM is a permutation and the output alphabet is binary. Each state is assigned a label which determines its output distribution. The states labeled 0 emit 0 with probability $0.5+\epsilon$ and the states labeled 1 emit 1 with probability $0.5+\epsilon$. We will randomly and uniformly choose the labels for the hidden states. Over the randomness in choosing the labels for the permutation, we will show that the expected error of the predictor $\mathcal{P}$ is large, which means that there must exist some permutation such that the predictor $\mathcal{P}$ incurs a high error. The rough proof idea is as follows. Say the Markov model is at hidden state $h_2$ at time 2, this is unknown to the predictor $\mathcal{P}$. The outputs for the first three time steps are $(x_0,x_1,x_2)$. The predictor $\mathcal{P}$ only looks at the outputs from time 0 to 2 for making the prediction for time 3. We show that with high probability over the choice of labels to the hidden states and the outputs $(x_0,x_1,x_2)$, the output $(x_0,x_1,x_2)$ from the hidden states $(h_0, h_1, h_2)$ is close in Hamming distance to the label of some other segment of hidden states, say $(h_4,h_5,h_6)$. Hence any predictor using only the past 3 outputs cannot distinguish whether the string $(x_0,x_1,x_2)$ was emitted by $(h_0,h_1,h_2)$ or $(h_4,h_5,h_6)$, and hence cannot make a good prediction for time 3 (we actually need to show that there are many segments like $(h_4,h_5,h_6)$ whose label is close to $(x_0,x_1,x_2)$). The proof proceeds via simple concentration bounds.

\appendix

\section{Proof of Theorem \ref{hmm}}\label{sec:hmm_app}

%We also show that the dependence on $c$ is optimal in Corollary \ref{0/1}. Consider a permutation on $n$ states, with the outputs being binary. Label each state with `+' or `-', the states labeled as `+' emit 1 with probability $0.5+1/c$ and the states labeled as `-' emit 0 with probability $0.5+1/c$. Consider the task of distinguishing a sequence $s_1$ of $\log n$ `+' states from a sequence $s_2$  with $\log n/2$ `+' states and $\log n/2$ `-' states. The expected number of 1 for sequence $s_1$ is 

%\hmm*
\medskip
%\noindent \textbf{Proposition~\ref{binary}.} 
\noindent\textbf{Theorem~\ref{thm:hmm}.} 
\emph{Suppose observations are generated by a Hidden Markov Model with at most $n$ hidden states, and output alphabet of size $d$. For $\epsilon>1/\log^{0.25}n$ there exists a window length $\ell = O(\frac{\log n}{\eps})$ and absolute constant $c$ such that for any $T \ge d^{c \ell},$ if $t \in \{1,2,\ldots,T\}$ is chosen uniformly at random, then the expected $\ell_1$ distance between the true distribution of $x_t$ given the entire history (and knowledge of the HMM), and the distribution predicted by the naive ``empirical'' $\ell$-th order Markov model based on $x_0,\ldots,x_{t-1}$, is bounded by $\sqrt{\eps}$.}
\medskip

\begin{proof}
	Let $\pi_t$ be a distribution over hidden states such that the probability of the $i$th hidden state under $\pi_t$ is the empirical frequency of the $i$th  hidden state from time $1$ to $t-1$ normalized by $(t-1)$. For $0\le s\le \ell-1$, consider the predictor $\mathcal{P}_t$ which makes a prediction for the distribution of observation $x_{t+s}$ given observations $x_{t},\dots,x_{t+s-1}$ based on the true distribution of $x_t$ under the HMM,  conditioned on the observations $x_{t},\dots,x_{t+s-1}$ and the distribution of the hidden state at time $t$ being $\pi_{t}$. We will show that in expectation over $t$,  $\mathcal{P}_t$ gets small  error averaged across the time steps $0\le s\le \ell-1$, with respect to the optimal prediction of $x_{t+s}$ with knowledge of the true hidden state $h_t$ at time $t$. In order to show this, we need to first establish that the true hidden state $h_t$ at time $t$ does not have very small probability under $\pi_t$, with high probability over the choice of $t$. 
	
	\begin{lemma}\label{lem:set}
		With probability $1-2/n$ over the choice of $t \in \{1,\dots, T\}$, the hidden state $h_t$ at time $t$ has probability at least $1/n^3$ under $\pi_t$.
	\end{lemma}
	\begin{proof}
		Consider the ordered set $\mathcal{S}_i$ of time indices $t$ where the hidden state $h_t=i$, sorted in increasing order. We first argue that picking a time step $t$ where the hidden state $h_t$ is a state $j$ which occurs rarely in the sequence is not very likely.  For  sets corresponding to hidden states $j$ which have probability less than $1/n^2$ under $\pi_T$, the cardinality $|\mathcal{S}_j|\le T/n^2$. The sum of the cardinality of all such small sets is at most $T/n$, and hence the probability that a uniformly random $t\in \{1,\dots, T\}$ lies in one of these sets is at most $1/n$.
		
		 Now consider the set of time indices $\mathcal{S}_i$  corresponding to some hidden state $i$ which has probability at least $1/n^2$ under $\pi_T$. For all $t$ which are not among the first $T/n^3$ time indices in this set, the hidden state $i$ has probability at least $1/n^3$ under $\pi_t$. We will refer to the first $T/n^3$ time indices in any set $S_i$ as the ``bad'' time steps for the hidden state $i$. Note that the fraction of the ``bad'' time steps corresponding to any hidden state which has probability at least $1/n^2$ under $\pi_T$ is at most $1/n$, and hence the total fraction of these ``bad'' time steps across all hidden states is at most $1/n$. Therefore using a union bound, with failure probability $2/n$,  the hidden state $h_t$ at time $t$ has probability at least $1/n^3$ under $\pi_t$.
		%Say we choose $t$ uniformly $t \in \{0,\dots, T\}$. Note that all time indices $t$ at which the hidden state $h_t$ has probability less than $1/n^2$ under $\pi_T$ are chosen with probability less than $1/n$ as 
	\end{proof}
	
	Consider any time index $t$, for simplicity   assume $t=0$, and let $OPT_s$ denote the conditional distribution of $x_s$ given observations $x_0,\ldots,x_{s-1}$, and knowledge of the hidden state at time $s=0$.  Let $M_s$ denote the conditional distribution of $x_s$ given only $x_0,\ldots,x_{s-1},$ given that the hidden state at time 0 has the distribution $\pi_0$. 
	
%	\begin{lemma} \label{lem:regret}
\medskip
\noindent \textbf{Lemma~\ref{lem:regret}.}
		\emph{For $\eps>1/n$, if the true hidden state at time 0 has probability at least $1/n^c$ under $\pi_0$, then for $\ell=c \log n/\eps^2$,  
		\begin{align}
		\E\Big[\frac{1}{\ell}\sum_{s=0}^{\ell-1} \| OPT_s - M_s \|_1 \Big] \le 4\eps ,\nonumber
		\end{align}
		where the expectation is with respect to the randomness in the outputs from time $0$ to $\ell-1$.}
%	\end{lemma}
\medskip
	
By Lemma \ref{lem:set}, for a randomly chosen $t\in\{1,\dots,T\}$ the probability that  the hidden state $i$ at time 0 has probability less than $1/n^3$ in the prior distribution $\pi_t$ is at most $2/n$. As the $\ell_1$ error at any time step can be at most 2, using Lemma \ref{lem:regret}, the expected average error of the predictor $\mathcal{P}_t$ across all $t$ is at most $4\epsilon +4/n\le 8\epsilon$ for $\ell=3\log n/\epsilon^2$. 

Now consider the predictor $\hat{\mathcal{P}}_t$ which for $0\le s\le \ell-1$ predicts  $x_{t+s}$ given  $x_{t},\dots,x_{t+s-1}$ according to the empirical distribution of $x_{t+s}$ given $x_{t},\dots,x_{t+s-1}$, based on the observations up to time $t$. We will now argue that the predictions of $\hat{\mathcal{P}}_t$ are close in expectation to the predictions of ${\mathcal{P}}_t$. Recall that prediction of ${\mathcal{P}}_t$ at time $t+s$ is the true distribution of $x_t$ under the HMM, conditioned on the observations $x_{t},\dots,x_{t+s-1}$ and the distribution of the hidden state at time $t$ being drawn from $\pi_{t}$. For any $s< \ell$, let $P_1$  refer to the prediction of $\hat{\mathcal{P}}_t$ at time $t+s$ and $P_2$ refer to the prediction of ${\mathcal{P}}_t$ at time $t+s$. We will show that $\norm{P_1-P_2}{1}$ is small in expectation over $t$. %We will show that for all $0\le s\le \ell$, the expected $\ell_1$ distance between $P_1$ --- the empirical distribution of $x_{t+s}$ given observations $x_{t},\dots,x_{t+s-1}$ and $P_2$ the distribution of $x_{t+s}$ conditioned on the observations $x_{t},\dots,x_{t+s-1}$ and the distribution of the hidden state at time $t$ being $\pi_{t}$ is small. 

We do this using a martingale concentration argument. Consider any string $r$ of length $s$. Let $Q_1(r)$ be the empirical probability of the string $r$ up to time $t$ and $Q_2(r)$ be the true probability of the string $r$ given that the hidden state at time $t$ is distributed as $\pi_t$. Our aim is to show that $|Q_1(r)-Q_2(r)|$ is small. Define the random variable $$Y_\tau = {\Pr[[x_{\tau}:x_{\tau+s-1}]=r|h_{\tau}]}-{I([x_{\tau}:x_{\tau+s-1}]=r)},$$ where $I$ denotes the indicator function and $Y_0$ is defined to be 0. We claim that $Z_\tau = \sum_{i=0}^{\tau} Y_i$ is a martingale with respect to the filtration $\{\phi\},\{h_1\},\{h_2,x_1\},\{h_3,x_2\},\dots, \{h_{t+1},x_t\}$. To verify, note that,
\begin{align*}
\E[Y_{\tau}|&\{h_1\},\{h_2,x_1\},\dots,\{h_{\tau},x_{\tau-1}\}] = \Pr[[x_{\tau}:x_{\tau+s-1}]=r|h_{\tau}] 
\\ &\quad -E[I([x_{\tau}:x_{\tau+s-1}]=r)|\{h_1\},\{h_2,x_1\},\dots,\{x_{\tau-1},h_{\tau}\}]\\
&= \Pr[[x_{\tau}:x_{\tau+s-1}]=r|h_{\tau}] - E[I([x_{\tau}:x_{\tau+s-1}]=r)|h_{\tau}]=0.
\end{align*}
Therefore $\E[Z_{\tau}|\{h_1\},\{h_2,x_1\},\dots,\{h_{\tau},x_{\tau-1}\}] = Z_{\tau-1}$, and hence $Z_\tau$ is a martingale. Also, note that $|Z_{\tau}-Z_{\tau-1}|\le 1$ as $0\le\Pr[[x_{\tau}:x_{\tau+s-1}]=r|h_{\tau}]\le 1$ and $0\le I([x_{\tau}:x_{\tau+s-1}]=r)\le 1$. Hence using Azuma's inequality (Lemma \ref{azuma}),
\begin{align*}
\Pr[|Z_{t-s}|\ge K]\le 2e^{-K^2/(2t)} .
\end{align*}
Note that $Z_{t-s}/(t-s)=Q_2(r)-Q_1(r)$. By Azuma's inequality and doing a union bound over all $d^s\le d^{\ell}$ strings $r$ of length $s$, for $c\ge4$ and $t\ge T/n^2=d^{c\ell}/n^2\ge d^{c\ell/2}$, we have $\norm{Q_1-Q_2}{1}\le 1/d^{c\ell/20}$ with failure probability at most $2d^{\ell}e^{-\sqrt{t}/2}\le 1/n^2$. Similarly, for all strings of length $s+1$, the estimated probability of the string has error at most $1/d^{c\ell/20}$ with failure probability $1/n^2$. As the conditional distribution of $x_{t+s}$ given observations $x_{t},\dots,x_{t+s-1}$ is the ratio of the joint distributions of $\{x_t,\dots,x_{t+s-1},x_{t+s}\}$ and $\{x_t,\dots,x_{t+s-1}\}$, therefore as long as the empirical distributions of the length $s$ and length $s+1$ strings are estimated with error at most $1/d^{c\ell/20}$ and the string $\{x_t,\dots,x_{t+s-1}\}$ has probability at least $1/d^{c\ell/40}$, the conditional distributions $P_1$ and $P_2$ satisfy $ \norm{P_1-P_2}{1}\le 1/n^2$. %As  $\norm{Q_1-Q_2}{1}\le 1/d^{c\ell/20}$ with failure probability $1/n^2$, for all outputs $x_{t},\dots,x_{t+s-1}$ which occur with probability at least $1/d^{c\ell/40}$ we have a $(1\pm 1/n^2)$ multiplicative approximation to the conditional distribution. 
By a union bound over all $d^s\le 	d^{\ell}$ strings and for $c\ge 100$, the total probability mass on strings which occur with probability less than $1/d^{c\ell/40}$ is at most $1/d^{c\ell/50}\le 1/n^2$ for $c\ge100$. Therefore $ \norm{P_1-P_2}{1}\le 1/n^2$ with overall failure probability $3/n^2$, hence the expected $\ell_1$ distance between $P_1$ and $P_2$ is at most $1/n$. 

By using the triangle inequality and the fact that the expected average error of ${\mathcal{P}}_t$ is at most $8\epsilon$ for $\ell=3\log n/\epsilon^2$, it follows that the expected average error of $\hat{\mathcal{P}}_t$ is at most $8\epsilon+1/n\le 7\epsilon$. Note that the expected average error of $\hat{\mathcal{P}}_t$ is the average of the expected errors of the empirical $s$-th order Markov models for $0\le s \le \ell-1$. Hence for $\ell=3\log n/\epsilon^2$ there must exist at least some $s< \ell$ such that the $s$-th order Markov model gets expected $\ell_1$ error at most $9\epsilon$.
%Using Azuma's inequality, $\Pr[Z_t\ge {t}^{0.25}]\le 2e^{-\sqrt{t}/2}$. Therefore, for $t\ge T/n^2$, $\Pr[Z_t\ge {t}^{0.25}]\le e^{-d^{c\ell/4}} \implies \Pr[|Q_1(r)-Q_2(r)|\ge {t}^{0.25}]\le e^{-d^{c\ell/4}}$ for $T=d^{c\ell}$. By a union bound over all $d^l\le d^{\ell}$ strings $r$ of length $l$, for $t\ge T/n^2$, $\norm{Q_1-Q_2}{1}\le 1/d^{c\ell/20}$ with probability $e^{-d^{c\ell/4}}$. 

\iffalse
Define the random variable $Y_\tau = \Pr[[x_{\tau}:x_{\tau+l-1}]=r|h_{\tau-1}]-I([x_{\tau}:x_{\tau+l-1}]=r)$, where $I$ denotes the indicator function. We claim that $Z_\tau = \sum_{i=1}^{\tau} Y_i$ is a martingale with respect to the filtration $\{h_1,x_1\},\{h_2,x_2\},\dots, \{h_t,x_t\}$. To verify, note that,
\begin{align*}
\E[Y_{\tau}|\{h_1,x_1\},\dots,\{h_{\tau-1},x_{\tau-1}\}] = \Pr[[x_{\tau}:x_{\tau+l-1}]=r|h_{\tau-1}] - E[I([x_{\tau}:x_{\tau+l-1}]=r)|\{h_1\},\dots,\{x_{\tau-1},h_{\tau-1}\}]
\end{align*}
Therefore $\E[Z_{\tau}||\{h_1,x_1\},\dots,\{h_{\tau-1},x_{\tau-1}\}] = Z_{\tau-1}$. Also, note that $Z_{\tau}$ 
\fi
	
	\subsection{Proof of Lemma \ref{lem:regret}}
	
	Let the prior for the distribution of the hidden states at time 0 be $\pi_0$. Let the true hidden state $h_0$ at time 0 be $1$ without loss of generality. We refer to the output at time $t$ by $x_s$. Let $H_0^s(i)=\Pr[h_0 = i | x_0^s ]$ be the posterior probability of the $i$th hidden state at time 0 after seeing the observations $x_0^s$ up to time $t$ and having the prior $\pi_0$ on the distribution of the hidden states at time 0. Let $u_s= H^s_0(1)$ and $v_s=1-u_s$. Define $P_i^s(j)=\Pr[x_s = j | x_0^{s-1}, h_0 = i]$ as the distribution of the output at time $t$ conditioned on the hidden state at time 0 being $i$ and observations $x_0^{s-1}$. Note that $OPT_s = P_1^s$. As before, define $R_{s}$ as the conditional distribution of $x_s$ given observations $x_0,\dotsb,x_{s-1}$ and initial distribution $\pi$ but \emph{not} being at hidden state $h_0$ at time 0 i.e. $R_s = (1/v_s) \sum_{i=2}^{n}H^s_0(i)P_i^s$. Note that $M_s$ is a convex combination of $OPT_s$ and $R_s$, i.e. $M_s = u_sOPT_s + v_s R_s$. Hence ${\norm{OPT_s-M_s}{1}}\le \;{\norm{OPT_s-R_s}{1}}$. Define $\delta_s = {\norm{OPT_s-M_s}{1}}$.
	
	Our proof relies on a martingale concentration argument, and in order to ensure that our martingale has bounded differences we will ignore outputs which cause a significant drop in the posterior of the true hidden state at time 0. Let $B$ be the set of all outputs $j$ at some time $t$ such that $\frac{OPT_s(j)}{R_s(j)} \le \frac{\epsilon^4}{c{\log n}}$. Note that, $\sum_{j\in B}^{}OPT_s(j) \le \frac{\epsilon^4\sum_{j\in B}^{}R_s(j)}{c{\log n}}\le \frac{\epsilon^4}{c{\log n}}$. Hence by a union bound, with failure probability at most $\epsilon^2$ any output $j$ such that $\frac{OPT_s(j)}{R_s(j)} \le \frac{\epsilon^4}{c{\log n}}$ is not emitted in a window of length $c{\log n}/\epsilon^2$. Hence we will only concern ourselves with sequences of outputs such that the output $j$ emitted at each step satisfies $\frac{OPT_s(j)}{R_s(j)} \le \frac{\epsilon^4}{c{\log n}}$, let the set of all such outputs be $\mathcal{S}_1$, note that $\Pr(x_0^s \notin \mathcal{S}_1) \le \epsilon^2$. Let $\E_{\mathcal{S}_1}[X]$ be the expectation of any random variable $X$ conditioned on the output sequence being in the set $\mathcal{S}_1$.
	
	Consider the sequence of random variables $X_s = \log u_s  - \log v_s $ for $s \in [-1,\ell-1]$. Let $X_{-1}=\log(\pi_1)-\log(1-\pi_1)$. Let $\Delta_{s+1}= X_{s+1}  - X_s $ be the change in $X_s$ on seeing the  output $x_{s+1}$ at time $s+1$. Let the output at time $s+1$ be $j$. We will first find an expression for $\Delta_{s+1}$. The posterior probabilities after seeing the $(s+1)$th output get updated according to Bayes rule,
	\begin{align*}
	H_0^{s+1}(1)&=\Pr[h_0 = 1 | x_0^s, x[s+1]=j] \\
	&= \frac{\Pr[h_0 = 1 | x_0^s ]\Pr[x[s+1] = j | h_0 = 1, x_0^s]}{\Pr[x[s+1] = j |x_0^s]}\\
	\implies u_{s+1} &= \frac{u_s OPT_{s+1}(j)}{\Pr[x[s+1] = j |x_0^s ]}.
	\end{align*}
	Let $\Pr[x[s+1] = j |x_0^s]=d_j$. Note that $H^{s+1}_0(i)=H_0^{s}(i)P_i^{s+1}(j)/d_j$ if the output at time $s+1$ is $j$. We can write,
	\begin{align}
	R_{s+1}&=\Big({\sum_{i=2}^{n}H_0^s(i)}P_i^{s+1}\Big)/{v_s} \nonumber\\ 
	v_{s+1} &= \sum_{i=2}^{n}H_0^{s+1}(i)=\Big(\sum_{i=2}^{n}H_0^s(i)P_i^{s+1}(j)\Big)/d_j \nonumber\\
	&= v_sR_{s+1}(j)/d_j.\nonumber
	\end{align}
	Therefore we can write $\Delta_{s+1}$ and its expectation $\E[\Delta_{s+1}]$ as,
	\begin{align}
	\Delta_{s+1} &= \log \frac{OPT_{s+1}(j)}{R_{s+1}(j)}\nonumber \\
	\implies \E[\Delta_{s+1}] &= \sum_{j}^{}OPT_{s+1}(j)\log \frac{OPT_{s+1}(j)}{R_{s+1}(j)} = D(OPT_{s+1}\parallel R_{s+1}). \nonumber
	\end{align}
	We define $\tilde{\Delta}_{s+1}$ as $\tilde{\Delta}_{s+1}:=\min\{\Delta_{s+1},\log\log n\}$ to keep martingale differences bounded. $\E[\tilde{\Delta}_{s+1}]$ then equals a truncated version of the KL-divergence which we define as follows.
	\begin{definition}\label{cool_Kl}
		For any two distributions $\mu(x)$ and $\nu(x)$, define the truncated KL-divergence as $\tilde{D}_C(\mu\parallel \nu) = \E\Big[\log \Big(\min\Big\{\mu(x)/\nu(x),C\Big\}\Big)\Big] $ for some fixed $C$.
	\end{definition}
	We are now ready to define our martingale. Consider the sequence of random variables $\tilde{X}_s := \tilde{X}_{s-1} + \tilde{\Delta}_s$ for $t\in[0,\ell-1]$, with $\tilde{X}_{-1}:=X_{-1}$. Define $\tilde{Z}_s := \sum_{s=1}^{n}\Big(\tilde{X}_s - \tilde{X}_{s-1} -\delta_{s}^2/2\Big)$. Note that $\Delta_s \ge \tilde{\Delta}_s \implies X_s \ge \tilde{X}_s$. 
	
	\begin{lemma}
		$\E_{\mathcal{S}_1}[\tilde{X}_{s} -\tilde{X}_{s-1} ] \ge \delta_{s}^2/2$, where the expectation is with respect to the output at time $t$. Hence the sequence of random variables $\tilde{Z}_s := \sum_{i=0}^{s}\Big(\tilde{X}_s-\tilde{X}_{s-1}-\delta_{s}^2/2\Big)$ is a submartingale with respect to the outputs.
	\end{lemma}
	\begin{proof}
		By definition $\tilde{X}_s - \tilde{X}_{s-1} = \tilde{\Delta}_s$ and $\E[\tilde{\Delta}_s] = \tilde{D}_{C}(OPT_{s}\parallel R_{s}), C= {\log n}$. By taking an expectation with respect to only sequences $\mathcal{S}_1$ instead of all possible sequences, we are removing events which have a negative contribution to $\E[\tilde{\Delta}_s]$, hence $$\E_{\mathcal{S}_1}[\tilde{\Delta}_s ] \ge \E[\tilde{\Delta}_s]= {\tilde{D}_{C}(OPT_{s}\parallel R_{s})}.$$ We can now apply Lemma \ref{cool_pinsker}. 
		\coolKL
		Hence $\E_{\mathcal{S}_1}[\tilde{\Delta}_s] \ge \frac{1}{2}\norm{OPT_{s}- R_{s}}{1}^2$. Hence $\E_{\mathcal{S}_1}[\tilde{X}_{s} -\tilde{X}_{s-1} ]\ge \delta_{s}^2/2$. 
	\end{proof}
	
	\begin{comment}
	In order to get a good upper bound on the submartingale, we make a small modification to it. Once $y_s$ goes down below $1/n^2$ we cap it at $1/n^2$ and fix it at $\frac{1}{n^2}$ for making the remaining predictions in the window. We also correspondingly cap $x_s$ and fix it at $1-\frac{1}{n^2}$. Define $X'_s = \log x'_s - \log y'_s$, where $y'_s$ and $x'_s$ are the capped versions of $x_s$ and $y_s$. Also, define $\epsilon'_{s} = \{ \epsilon_s \;\text{if} \;\epsilon_s \ge \frac{1}{n}\; \text{and 0 otw}\}$. We prove that the new sequence of random variable $Z'_s = \sum_{i=1}^{s}(X'_s-X'_{s-1}-{\epsilon'_{s}}^2/2)$ is also a submartingale with respect to the outputs.
	
	\begin{lemma}
	$\E[X'_{s+1}-X'_{s}] \ge {\epsilon'_{s+1}}^2/2$, where the expectation is with respect to the output at time $s+1$. Hence the sequence of random variables $Z'_s$ is a submartingale with respect to the outputs.
	\end{lemma}
	\begin{proof}
	Note that $\E[X'_{s+1}-X'_{s}] \ge 0$, as $p\log x+ (1-p)\log (1-x)$ is maximized at $x=p$. Also, note that 
	\end{proof}
	\end{comment}
	
	We now claim that our submartingale has bounded differences.
	
	\begin{lemma}
		$|\tilde{Z}_{s}-\tilde{Z}_{s-1}|\le \sqrt{2}\log (c{\log n}/\eps^4)$.
	\end{lemma}
	\begin{proof}
		Note that $(\delta_{s}^2-\delta_{s-1}^2)/2$ can be at most 2. $Z_{s}-Z_{s-1} =  \tilde{\Delta}_s$. By definition $\tilde{\Delta}_s \le \log( {\log n})$. Also, $\tilde{\Delta}_s \ge -\log(c {\log n}/\epsilon^4)$ as we restrict ourselves to sequences in the set $\mathcal{S}_1$. Hence $|\tilde{Z}_{s}-\tilde{Z}_{s-1}|\le \log (c{\log n}/\epsilon^4) + 2\le \sqrt{2}\log (c{\log n}/\epsilon^4)$. 
	\end{proof}
	
	We now apply Azuma-Hoeffding to get submartingale concentration bounds. 
	%\textbf{Azuma-Hoeffding inequality- }
	%\emph{
	%	Let $Z_i$ be a submartingale and let $|Z_{i}-Z_{i-1}|\le c$. Then $\Pr[Z_s-Z_0\le -\lambda] \le \exp\Big(\frac{-\lambda^2}{2tc^2}\Big)$
	%}
	\begin{lemma}\label{azuma}
		({Azuma-Hoeffding inequality}) Let $Z_i$ be a submartingale with $|Z_{i}-Z_{i-1}|\le C$. Then $\Pr[Z_s-Z_0\le -\lambda] \le \exp\Big(\frac{-\lambda^2}{2tC^2}\Big)$
	\end{lemma}
	Applying Lemma \ref{azuma} we can show,
	\begin{align}
	\Pr[\tilde{Z}_{\ell-1}-\tilde{Z}_0&\le -{c\log n}] \le \exp\Big(\frac{-{c\log n}}{4(1/\epsilon)^2\log^2(c{\log n}/\epsilon^4)}\Big) \le \epsilon^2, \label{eq:azuma_app}
	\end{align}
	for $\eps \ge 1/\log^{0.25} n$ and $c\ge 1$.
	We now bound the average error in the window 0 to $\ell-1$. With failure probability at most $\epsilon^2$ over the randomness in the outputs, $\tilde{Z}_{\ell-1}-\tilde{Z}_0\ge -{c\log n}$ by Eq. \ref{eq:azuma_app}. Let $\mathcal{S}_2$ be the set of all sequences in $\mathcal{S}_1$ which satisfy $\tilde{Z}_{\ell-1}-\tilde{Z}_0\ge -{c\log n}$. Note that $X_0 = \tilde{X}_0 \ge \log(1/\pi_1)$. Consider the last point after which $v_s$ decreases below $\epsilon^2$ and remains below that for every subsequent step in the window. Let this point be $\tau$, if there is no such point define $\tau$ to be $\ell-1$. The total contribution of the error at every step after the $\tau$th step to the average error is at most a $\epsilon^2$ term as the error after this step is at most $\epsilon^2$. Note that $X_{\tau} \le \log (1/\epsilon)^2 \implies \tilde{X}_{\tau}\le \log (1/\epsilon)^2$ as $ \tilde{X}_{s}\le X_{s}$. Hence for all sequences in $\mathcal{S}_2$,
	\begin{align*}
	\tilde{X}_\tau &\le \log (1/\epsilon)^2 \nonumber \\
	\implies \tilde{X}_{\tau}  - \tilde{X}_{-1} &\le \log (1/\epsilon)^2+\log(1/\pi_1)\\
	\stackrel{(a)}{\implies} 0.5\sum_{s=0}^{\tau}{\delta}_{s}^2 &\le 2\log n+\log(1/\pi_1) + c{\log n}\\
	\stackrel{(b)}{\implies} 0.5\sum_{s=0}^{\tau}{\delta}_{s}^2 &\le 2(c+1)\log n \le 4c\log n \\
	%\implies \sum_{s=1}^{L}\epsilon_s(s_0^s)^2 &\le 4\log n + o(1)\\
	\stackrel{(c)}{\implies} \frac{\sum_{s=0}^{\ell-1}{\delta}_{s}^2}{c\log n/\epsilon^2} &\le 8\epsilon^2  \\
	\stackrel{(c)}{\implies} \frac{{\sum_{s=0}^{\ell-1}{\delta}_{s}}}{c\log n/\epsilon^2} &\le 3\epsilon, 
	\end{align*}
	where (a) follows by Eq. \ref{eq:azuma_app}, and as $\epsilon\ge 1/n$; (b) follows as $ \log(1/\pi_1)\le c\log n$, and $c\ge 1$; (c) follows because $\log(1/\pi_1)\le c\log n)$; and (d) follows from Jensen's inequality. As the total probability of sequences outside $\mathcal{S}_2$ is at most $2\epsilon^2$, $\E[{\sum_{s=0}^{\ell-1}{\delta}_{s}}] \le 4\epsilon$, whenever the hidden state $i$ at time 0 has probability at least $1/n^c$ in the prior distribution $\pi_0$.  %Hence the total expected error is at most $4\epsilon\ell$ except with probability $1/n^{c-1}$.

\end{proof}

\subsection{Proof of Modified Pinsker's Inequality (Lemma \ref{cool_pinsker}) }

\coolKL*
\begin{proof}
	We rely on the following Lemma which bounds the KL-divergence for binary distributions-
	\begin{lemma}
		For every $0 \le q \le p\le 1$, we have
		\begin{enumerate}
			\item $p\log \frac{p}{q} + (1-p)\log \frac{1-p}{1-q} \ge 2(p-q)^2$.
			\item $3p + (1-p)\log \frac{1-p}{1-q} \ge 2(p-q)^2$.
		\end{enumerate}
	\end{lemma}
	\begin{proof}
		For the second result, first observe that $\log(1/(1-q))\ge 0$ and $(p-q)\le p$ as $q\le p$. Both the results then follow from standard calculus.
	\end{proof}
	Let $A := \{x \in X : \mu(x) \ge \nu(x)\}$ and $B := \{x \in X : \mu(x) \ge C\nu(x)\}$. Let $\mu(A)=p$, $\mu(B)=\delta,  \nu(A)=q$ and $\nu(B)=\epsilon$. Note that $\norm{\mu-\nu}{1}=2(\mu(A)-\nu(A))$. By the log-sum inequality--
	\begin{comment}
	\begin{align}
	D(\mu \parallel \nu) &= \sum_{x\in A}\mu(x)\log\frac{\mu(x)}{\nu(x)} + \sum_{x\in X-A}\mu(x)\log\frac{\mu(x)}{\nu(x)} \\
	&\ge \mu(A)\log\frac{\mu(A)}{\nu(A)}+\mu(X-A)\log\frac{\mu(X-A)}{\nu(X-A)}\\
	&\ge 2(\mu(A)-\nu(A))^2\\
	&= \frac{1}{2}\norm{\mu-\nu}{1}^2
	\end{align}
	\end{comment}
	\begin{align}
	\tilde{D}_C(\mu \parallel \nu) &= \sum_{x\in B}\mu(x)\log\frac{\mu(x)}{\nu(x)} + \sum_{x\in A-B}\mu(x)\log\frac{\mu(x)}{\nu(x)} + \sum_{x\in X-A}\mu(x)\log\frac{\mu(x)}{\nu(x)} \nonumber\\
	&= \delta\log C+ (p-\delta)\log\frac{p-\delta}{q-\epsilon}+(1-p)\log \frac{1-p}{1-q}. \nonumber
	\end{align}
	\begin{enumerate}
		\item \emph{Case 1}: $0.5 \le \frac{\delta}{p} \le 1$
		\begin{align}
		\tilde{D}_C(\mu \parallel \nu) &\ge \frac{p}{2}\log C +(1-p)\log \frac{1-p}{1-q}\nonumber \\
		&\ge 2(p-q)^2= \frac{1}{2}\norm{\mu-\nu}{1}^2. \nonumber
		\end{align}
		\item \emph{Case 2}: $\frac{\delta}{p}<0.5 $
		\begin{align}
		\tilde{D}_C(\mu \parallel \nu) &= \delta\log C+ (p-\delta)\log\frac{p}{q-\epsilon}+(p-\delta)\log\Big(1-\frac{\delta}{p}\Big)+(1-p)\log \frac{1-p}{1-q}\nonumber \\
		&\ge \delta\log C+ (p-\delta)\log\frac{p}{q}-(p-\delta)\frac{2\delta}{p}+(1-p)\log \frac{1-p}{1-q} \nonumber\\
		&\ge \delta(\log C-2)+ (p-\delta)\log\frac{p}{q}+(1-p)\log \frac{1-p}{1-q}.\nonumber 
		\end{align}
		\begin{enumerate}
			\item \emph{Sub-case 1}: $\log \frac{p}{q} \ge 6$
			\begin{align}
			\tilde{D}_C(\mu \parallel \nu) &\ge (p-\delta)\log\frac{p}{q}+(1-p)\log \frac{1-p}{1-q}\nonumber\\
			&\ge 3p+(1-p)\log \frac{1-p}{1-q} \nonumber\\
			&\ge 2(p-q)^2= \frac{1}{2}\norm{\mu-\nu}{1}^2. \nonumber
			\end{align}
			\item \emph{Sub-case 2}: $\log \frac{p}{q} < 6$
			\begin{align}
			\tilde{D}_C(\mu \parallel \nu) &\ge \delta(\log C-2-\log\frac{p}{q})+ p\log\frac{p}{q}+(1-p)\log \frac{1-p}{1-q}\nonumber \\
			&\ge 2(p-q)^2= \frac{1}{2}\norm{\mu-\nu}{1}^2.\nonumber
			\end{align}
		\end{enumerate}
	\end{enumerate}
\end{proof}
\section{Proof of Lower Bound for Large Alphabets}\label{sec:largen_app}

\subsection{CSP formulation}

We first go over some notation that we will use for CSP problems, we follow the same notation and setup as in \citet{feldman2015complexity}. Consider the following model for generating a random CSP instance on $n$ variables with a satisfying assignment $\boldsymbol{\sigma}$. The $k$-CSP is defined by the predicate $P:\{0,1\}^k \rightarrow \{0,1\}$. We represent a $k$-clause by an ordered $k$-tuple of literals from $\{x_1, \dotsb, x_n, \bar{x}_1, \dotsb, \bar{x}_n\}$ with no repetition of variables and let $X_k$ be the set of all such $k$-clauses. For a $k$-clause $C=(l_1, \dotsb, l_k)$ let $\boldsymbol{\sigma}(C) \in \{0,1\}^k$ be the $k$-bit string of values assigned by $\boldsymbol{\sigma}$ to literals in $C$, that is $\{\boldsymbol{\sigma}(l_1), \dotsb, \boldsymbol{\sigma}(l_k)\}$ where $\boldsymbol{\sigma}(l_i)$ is the value of the literal $l_i$ in assignment $\boldsymbol{\sigma}$. In the planted model we draw clauses with probabilities that depend on the value of $\boldsymbol{\sigma}(C)$. Let $Q:\{0,1\}^k\rightarrow \mathbb{R}^+, \sum_{\textbf{t} \in \{0,1\}^k}^{}Q(\textbf{t})=1$ be some distribution over satisfying assignments to $P$. The distribution $Q_{\boldsymbol{\sigma}}$ is then defined as follows-
\begin{align}
	Q_{\boldsymbol{\sigma}}(C)=\frac{Q(\boldsymbol{\sigma}(C))}{\sum_{C'\in X_k}^{}Q(\boldsymbol{\sigma}(C'))}\label{csp}
\end{align}
Recall that for any distribution $Q$ over satisfying assignments we define its complexity $r$ as the largest $r$ such that the distribution $Q$ is $(r-1)$-wise uniform (also referred to as $(r-1)$-wise independent in the literature) but not $r$-wise uniform. \\

Consider the CSP $\mathcal{C}$ defined by a collection of predicates $P(\textbf{y})$ for each $\textbf{y} \in \{0,1\}^m$ for some $m\le k/2$. Let $\textbf{A}\in \{0,1\}^{m \times k}$ be a matrix with full row rank over the binary field. We will later choose $\textbf{A}$ to ensure the CSP has high complexity. For each $\textbf{y}$, the predicate ${P}(\textbf{y})$ is the set of solutions to the system ${\textbf{y}=\textbf{Av} \mod 2}$ where $\textbf{v}=\boldsymbol{\sigma}(C)$. For all $\textbf{y}$ we define $Q_{\textbf{y}}$ to be the uniform distribution over all consistent assignments, i.e. all  $\textbf{v}\in\{0,1\}^k$ satisfying ${\textbf{y}=\textbf{Av} \mod 2}$. The planted distribution $Q_{\boldsymbol{\sigma},\textbf{y}}$ is defined based on $Q_{\textbf{y}}$ according to Eq. \ref{csp}. Each clause in $\mathcal{C}$ is chosen by first picking a $\textbf{y}$ uniformly at random and then a clause from the distribution $Q_{\boldsymbol{\sigma},\textbf{y}}$. For any planted $\boldsymbol{\sigma}$ we define $Q_{\boldsymbol{\sigma}}$ to be the distribution over all consistent clauses along with their labels $\textbf{y}$. Let $U_k$ be the uniform distribution over $k$-clauses, with each clause assigned a uniformly chosen label $\textbf{y}$. Define $Q_{\boldsymbol{\sigma}}^{\eta}=(1-\eta)Q_{\boldsymbol{\sigma}}+\eta U_k$, for some fixed noise level $\eta>0$. We consider $\eta$ to be a small constant less than 0.05. This corresponds to adding noise to the problem by mixing the planted and the uniform clauses. The problem gets harder as $\eta$ becomes larger, for $\eta=0$ it can be efficiently solved using Gaussian Elimination. \\

We will define another CSP $\mathcal{C}_0$ which we show reduces to $\mathcal{C}$ and for which we can obtain hardness using Conjecture 1. The label $\textbf{y}$ is fixed to be the all zero vector in $\mathcal{C}_0$. Hence $Q_{0}$, the distribution over satisfying assignments for $\mathcal{C}_0$, is the uniform distribution over all vectors in the null space of $\textbf{A}$ over the binary field. We refer to the planted distribution in this case as $Q_{\boldsymbol{\sigma},{{0}}}$. Let $U_{k,0}$ be the uniform distribution over $k$-clauses, with each clause now having the label ${0}$. For any planted assignment $\boldsymbol{\sigma}$, we denote the distribution of consistent clauses of $\mathcal{C}_0$ by $Q_{\boldsymbol{\sigma},{0}}$. As before define $Q_{\boldsymbol{\sigma},0}^{\eta}=(1-\eta)Q_{\boldsymbol{\sigma},0}+\eta U_{k,0}$ for the same $\eta$.\\

Let $L$ be the problem of distinguishing between $U_k$ and $Q_{\boldsymbol{\sigma}}^{\eta}$ for some randomly and uniformly chosen $\boldsymbol{\sigma} \in \{0,1\}^n$ with success probability at least $2/3$. Similarly, let $L_0$ be the problem of distinguishing between $U_{k,0}$ and $Q_{\boldsymbol{\sigma},0}^{\eta}$ for some randomly and uniformly chosen $\boldsymbol{\sigma} \in \{0,1\}^n$ with success probability at least $2/3$. $L$ and $L_0$ can be thought of as the problem of distinguishing random instances of the CSPs from instances with a high value. Note that $L$ and $L_0$ are at least as hard as the problem of refuting the random CSP instances $U_k$ and $U_{k,0}$, as this corresponds to the case where $\eta=0$. We claim that an algorithm for $L$ implies an algorithm for $L_0$.

\begin{restatable}{lemma}{cspred}\label{csp_red}
	If $L$ can be solved in time $t(n)$ with $s(n)$ clauses, then $L_0$ can be solved in time $O(t(n)+s(n))$ and $s(n)$ clauses.
\end{restatable}

Let the complexity of $Q_{{0}}$ be $\gamma k$, with $\gamma \ge 1/10$ (we demonstrate how to achieve this next). By Conjecture 1 distinguishing between $U_{k,0}$ and $Q_{\boldsymbol{\sigma},0}^{\eta}$ requires at least $\tilde{\Omega}(n^{\gamma k/2})$ clauses. We now discuss how $\textbf{A}$ can be chosen to ensure that the complexity of $Q_{{0}}$ is $\gamma k$.

%$Q_{\boldsymbol{\sigma}}^{\eta}=(1-\eta)Q_{\sigma}+\eta U_k$ 

\subsection{Ensuring High Complexity of the CSP}\label{subsec:coding}

Let $\mathcal{N}$ be the null space of $\textbf{A}$. Note that the rank of $\mathcal{N}$ is $(k-m)$. For any subspace $\mathcal{D}$, let $\textbf{w}(\mathcal{D})=(w_1, w_2, \dotsb, w_k)$ be a randomly chosen vector from $\mathcal{D}$. To ensure that $Q_{0}$ has complexity  $\gamma k$, it suffices to show that the random variables $\textbf{w}(\mathcal{N})= (w_1, w_2, \dotsb, w_k)$ are $(\gamma k-1)$-wise uniform. We use the theory of error correcting codes to find such a matrix $\textbf{A}$.\\

A binary linear code $\mathcal{B}$ of length $k$ and rank $m$ is a linear subspace of $\mathbb{F}_2^k$ (our notation is different from the standard notation in the coding theory literature to suit our setting). The rate of the code is defined to be $m/k$. The generator matrix of the code is the matrix $\mathbf{G}$ such that ${\mathcal{B}=\{\textbf{Gv}, \textbf{v}\in \{0,1\}^m\}}$. The parity check matrix of the code is the matrix $\textbf{H}$ such that $\mathcal{B}=\{\textbf{c}\in \{0,1\}^k : \textbf{Hc}=0 \}$. The distance $d$ of a code is the weight of the minimum weight codeword and the relative distance  $\delta$ is defined to be $\delta=d/k$. For any codeword $\mathcal{B}$ we define its dual codeword $\mathcal{B}^{T}$ as the codeword with generator matrix $\mathbf{H}^T$ and parity check matrix $\mathbf{G}^{T}$. Note that the rank of the dual codeword of a code with rank $m$ is $(k-m)$. We use the following standard result about linear codes--

\begin{fact}
	If $\mathcal{B}^T$ has distance $l$, then $\textbf{w}(\mathcal{B})$ is $(l-1)$-wise uniform.
\end{fact}  

Hence, our job of finding $\textbf{A}$ reduces to finding a dual code with distance $\gamma k$ and rank $m$, where $\gamma=1/10$ and $m\le k/2$. We use the Gilbert-Varshamov bound to argue for the existence of such a code. Let $H(p)$ be the binary entropy of $p$.

\begin{lemma}
	{(Gilbert-Varshamov bound)} For every $0\le \delta <1/2$, and $0<\epsilon \le 1-H(\delta)$, there exists a code with rank $m$ and relative distance $\delta$ if $m/k = 1-H(\delta)-\epsilon$.
\end{lemma}

Taking $\delta = 1/10$, $H(\delta)\le 0.5$, hence there exists a code $\mathcal{B}$ whenever $m/k \le 0.5$, which is the setting we're interested in. We choose $\textbf{A}=\textbf{G}^T$, where $\textbf{G}$ is the generator matrix of $\mathcal{B}$. Hence the null space of $\mathbf{A}$ is $(k/10-1)$-wise uniform, hence the complexity of $Q_{{0}}$ is $\gamma k$ with $\gamma\ge1/10$. Hence for all $k$ and $m\le k/2$ we can find a $\textbf{A} \in \{0,1\}^{m\times k}$ to ensure that the complexity of $Q_{{0}}$ is $\gamma k$.

\subsection{Sequential Model of CSP and Sample Complexity Lower Bound}

We now construct a sequential model which derives hardness from the hardness of $L$. Here we slightly differ from the outline presented in the beginning of Section \ref{sec:lower1} as we cannot base our sequential model directly on $L$ as generating random $k$-tuples without repetition increases the mutual information, so we formulate a slight variation $L'$ of $L$ which we show is at least as hard as $L$. We did not define our CSP instance allowing repetition as that is different from the setting examined in \citet{feldman2015complexity}, and hardness of the setting with repetition does not follow from hardness of the setting allowing repetition, though the converse is true.

\subsubsection{Constructing sequential model}\label{subsubsec:seq_constuct}

Consider the following family of sequential models $\mathcal{R}(n,\mathbf{A}_{m \times k})$ where $\textbf{A} \in \{0,1\}^{m\times k}$ is chosen as defined previously. The output alphabet of all models in the family is $\mathcal{X}=\{a_i, 1\le i\le 2n\}$ of size $2n$, with $2n/k$ even. We choose a subset $\mathcal{S}$ of $\mathcal{X}$ of size $n$, each choice of $\mathcal{S}$ corresponds to a model $\mathcal{M}$ in the family. Each letter in the output alphabet is encoded as a 1 or 0 which represents whether or not the letter is included in the set $\mathcal{S}$, let $\textbf{u}\in \{0,1\}^{2n}$ be the vector which stores this encoding so $u_i= 1$ whenever the letter $a_i$ is in $\mathcal{S}$. Let $\boldsymbol{\sigma} \in \{0,1\}^{n}$ determine the subset $\mathcal{S}$ such that entry $u_{2i-1}$ is 1 and $u_{2i}$ is 0 when $\boldsymbol{\sigma}_i$ is 1 and $u_{2i-1}$ is 0 and $u_{2i}$ is 1 when $\boldsymbol{\sigma}_i$ is 0, for all $i$. We choose $\boldsymbol{\sigma}$ uniformly at random from $\{0,1\}^{n}$ and each choice of $\boldsymbol{\sigma}$ represents some subset $\mathcal{S}$, and hence some model $\mathcal{M}$.  We partition the output alphabet $\mathcal{X}$ into $k$ subsets of size $2n/k$ each so the first $2n/k$ letters go to the first subset, the next $2n/k$ go to the next subset and so on. Let the $i$th subset be $\mathcal{X}_i$. Let $\mathcal{S}_i$ be the set of elements in $\mathcal{X}_i$ which belong to the set $\mathcal{S}$.\\

At time 0, $\mathcal{M}$ chooses $\textbf{v}\in \{0,1\}^k$ uniformly at random from $\{0,1\}^k$. %$v_i = u_{s(i)}$ where $s(i)$ is the character appearing at the $i$th time step, hence $\textbf{v}$ denotes a mapping from the outputs from time $1$ to $k$ to $\{0,1\}^k$. 
At time $i, i \in \{0, \dotsb, k-1\}$, if $v_i=1$, then the model chooses a letter uniformly at random from the set $\mathcal{S}_i$, otherwise if $v_i=0$ it chooses a letter uniformly at random from $\mathcal{X}_i-\mathcal{S}_i$. With probability $(1-\eta)$ the outputs for the next $m$ time steps from $k$ to $(k+m-1)$ are $\textbf{y}=\textbf{Av} \mod 2$, with probability $\eta$ they are $m$ uniform random bits. The model resets at time $(k+m-1)$ and repeats the process.\\

Recall that $I(\mathcal{M})$ is at most $m$ and  $\mathcal{M}$ can be simulated by an HMM with $2^{m}(2k+m)+m$ hidden states (see Section \ref{subsec:sketch}). \\

\subsubsection{Reducing sequential model to CSP instance}

We reveal the matrix $\textbf{A}$ to the algorithm (this corresponds to revealing the transition matrix of the underlying HMM), but the encoding $\boldsymbol{\sigma}$ is kept secret. The task of finding the encoding $\boldsymbol{\sigma}$ given samples from $\mathcal{M}$ can be naturally seen as a CSP. Each sample is a clause with the literal corresponding to the output letter $a_i$ being $x_{ (i+1)/2}$ whenever $i$ is odd and $\bar{x}_{ i/2 }$ when $i$ is even. We refer the reader to the outline at the beginning of the section for an example. We denote $\mathcal{C}'$ as the CSP $\mathcal{C}$ with the modification that the $i$th literal of each clause is the literal corresponding to a letter in $\mathcal{X}_i$ for all $1\le i \le k$. Define $Q_{\boldsymbol{\sigma}}'$ as the distribution of consistent clauses for the CSP $\mathcal{C}'$. Define $U'_k$ as the uniform distribution over $k$-clauses with the additional constraint that the $i$th literal of each clause is the literal corresponding to a letter in $\mathcal{X}_i$ for all $1\le i \le k$. Define $Q_{\boldsymbol{\sigma}}^{'\eta}=(1-\eta)Q'_{\boldsymbol{\sigma}} + \eta U'_k$. Note that samples from the model $\mathcal{M}$ are equivalent to clauses from $Q_{\boldsymbol{\sigma}}^{'\eta}$. We show that hardness of $L'$ follows from hardness of $L$--

\begin{restatable}{lemma}{cspredii}\label{hardness_L'}
	If $L'$ can be solved in time $t(n)$ with $s(n)$ clauses, then $L$ can be solved in time $t(n)$ with $O(s(n))$ clauses. Hence if Conjecture 1 is true then $L'$ cannot be solved in polynomial time with less than $\tilde{\Omega}(n^{\gamma k/2})$ clauses. 
\end{restatable}

We can now prove the Theorem \ref{high_n} using Lemma \ref{hardness_L'}.

\highn*
\begin{proof}
	We describe how to choose the family of sequential models $\mathcal{R}(n, \textbf{A}_{m \times k})$ for each value of $\epsilon$ and $T$. Recall that the HMM has $T=2^m(2k+m)+m$ hidden states. Let $T'=2^{m+2}(k+m)$. Note that $T'\ge T$. Let $t = \log T'$. We choose $m=t-\log(1/\epsilon)-\log (t/5)$, and $k$ to be the solution of $t=m+\log(k+m)+2$, hence $k=t/(5\epsilon) - m-2$. Note that for $\epsilon\le 0.1$, $k\ge m$. Let $\epsilon'=\frac{2}{9}\frac{m}{k+m}$. We claim $\epsilon\le \epsilon'$. To verify, note that $k+m=t/(5\epsilon)-2$. Therefore,
	\[
	\epsilon'=\frac{2m}{9(k+m)} = \frac{10\epsilon(t-\log(1/\epsilon)-\log (t/5))}{9t(1-10\epsilon/t)} \ge \epsilon,
	\]
	for sufficiently large $t$ and $\epsilon\ge 2^{-ct}$ for a fixed constant $c$. Hence proving hardness for obtaining error $\epsilon'$ implies hardness for obtaining error $\epsilon$. We choose the matrix $\textbf{A}_{m \times k}$ as outlined earlier. For each vector $\boldsymbol{\sigma}\in\{0,1\}^n$ we define the family of sequential models $\mathcal{R}(n,\textbf{A})$ as earlier. Let $\mathcal{M}$ be a randomly chosen model in the family.
	
	We first show the result for the relative zero-one loss. The idea is that any algorithm which does a good job of predicting the outputs from time $k$ through $(k+m-1)$ can be used to distinguish between instances of the CSP with a high value and uniformly random clauses. This is because it is not possible to make good predictions on uniformly random clauses. We relate the zero-one error from time $k$ through $(k+m-1)$ with the relative zero-one error from time $k$ through $(k+m-1)$ and the average zero-one error for all time steps to get the required lower bounds.\\
	
	Let $\rho_{01}(\mathcal{A})$ be the average zero-one loss of some polynomial time algorithm $\mathcal{A}$ for the output time steps $k$ through $(k+m-1)$ and $\delta_{01}'(\mathcal{A})$ be the average relative zero-one loss of $\mathcal{A}$ for the output time steps $k$ through $(k+m-1)$ with respect to the optimal predictions. For the distribution $U_k'$ it is not possible to get $\rho_{01}(\mathcal{A}) <0.5$ as the clauses and the label $\textbf{y}$ are independent and $\textbf{y}$ is chosen uniformly at random from $\{0,1\}^m$. For $Q_{\boldsymbol{\sigma}}^{'\eta}$ it is information theoretically possible to get $\rho_{01}(\mathcal{A})=\eta/2$. Hence any algorithm which gets error $\rho_{01}(\mathcal{A}) \le 2/5$ can be used to distinguish between $U_k'$ and $Q_{\boldsymbol{\sigma}}^{'\eta}$. Therefore by Lemma \ref{hardness_L'} any polynomial time algorithm which gets $\rho_{01}(\mathcal{A})\le 2/5$ with probability greater than $2/3$ over the choice of $\mathcal{M}$ needs at least $\tilde{\Omega}(n^{\gamma k/2})$ samples. Note that $\delta_{01}'(\mathcal{A})=\rho_{01}(\mathcal{A})-\eta/2$. As the optimal predictor $\mathcal{P}_{\infty}$ gets $\rho_{01}(\mathcal{P}_{\infty})=\eta/2<0.05$, therefore  $\delta_{01}'(\mathcal{A})\le 1/3 \implies \rho_{01}(\mathcal{A})\le 2/5$. Note that $\delta_{01}(\mathcal{A}) \ge \delta_{01}'(\mathcal{A})\frac{m}{k+m}$. This is because $\delta_{01}(\mathcal{A})$ is the average error for all $(k+m)$ time steps, and the contribution to the error from time steps $0$ to $(k-1)$ is non-negative. Also, $\frac{1}{3} \frac{m}{k+m} > {\epsilon'}$, therefore, $\delta_{01}(\mathcal{A})  < {\epsilon'} \implies \delta_{01}'(\mathcal{A})< \frac{1}{3} \implies \rho_{01}(\mathcal{A}) \le 2/5$. Hence any polynomial time algorithm which gets average relative zero-one loss less than $ {\epsilon'}$ with probability greater than $2/3$ needs at least $\tilde{\Omega}(n^{\gamma k/2})$ samples. The result for $\ell_1$ loss follows directly from the result for relative zero-one loss, we next consider the KL loss.\\
	
	Let $\delta_{KL}'(\mathcal{A})$ be the average KL error of the algorithm $\mathcal{A}$ from time steps $k$ through $(k+m-1)$. By application of Jensen's inequality and Pinsker's inequality, $\delta_{KL}'(\mathcal{A}) \le 2/9 \implies \delta_{01}'(\mathcal{A}) \le1/3$. Therefore, by our previous argument any algorithm which gets  $\delta_{KL}' (\mathcal{A})<2/9$ needs $\tilde{\Omega}(n^{\gamma k/2})$ samples. But as before, $\delta_{KL} (\mathcal{A})\le{\epsilon'} \implies \delta'_{KL} (\mathcal{A})\le 2/9$.  Hence any polynomial time algorithm which succeeds with probability greater than $2/3$ and gets average KL loss less than ${\epsilon'}$ needs at least $\tilde{\Omega}(n^{\gamma k/2})$ samples.\\
	 
	We lower bound $k$ by a linear function of $\log T/\epsilon$ to express the result directly in terms of $\log T/\epsilon$. We claim that $\log T/\epsilon$ is at most $10k$. This follows because--
	\begin{align}
\log T/\epsilon \le t/\epsilon = 5(k+m) +10\le 15k \nonumber
	\end{align}
	Hence any polynomial time algorithm needs $n^{\Theta(\log T/\epsilon)}$ samples to get average relative zero-one loss, $\ell_1$ loss, or KL loss less than $\epsilon$ on $\mathcal{M}$.\\

	\end{proof}
	
	\subsection{Proof of Lemma \ref{csp_red}}
	
	\cspred*
	
	\begin{proof}
		We show that a random instance of $\mathcal{C}_0$ can be transformed to a random instance of $\mathcal{C}$ in time $s(n)O(k)$ by independently transforming every clause $C$ in $\mathcal{C}_0$ to a clause $C'$ in $\mathcal{C}$ such that $C$ is satisfied in the original CSP $\mathcal{C}_0$ with some assignment $\textbf{t}$ to $\textbf{x}$ if and only if the corresponding clause $C'$ in $\mathcal{C}$ is satisfied with the same assignment $\textbf{t}$ to $\textbf{x}$. For every $\textbf{y} \in \{0,1\}^m$ we pre-compute and store a random solution of the system ${\textbf{y}=\textbf{Av} \mod 2}$, let the solution be ${v}({\textbf{y}})$. Given any clause $C=(x_1, x_2, \dotsb, x_k)$ in $\mathcal{C}_0$, choose $\textbf{y}\in\{0,1\}^m$ uniformly at random. We generate a clause $C'=(x'_1, x'_2, \dotsb, x'_k)$ in $\mathcal{C}$ from the clause $C$ in $\mathcal{C}_0$ by choosing the literal $x_i'= \bar{x}_i$ if $v_i(\textbf{y})=1$ and $x_i'= x_i$ if $v_i(\textbf{y})=0$. By the linearity of the system, the clause $C'$ is a consistent clause of $\mathcal{C}$ with some  assignment $\textbf{x}=\textbf{t}$ if and only if the clause $C$ was a consistent clause of $\mathcal{C}_0$ with the same assignment $\textbf{x}=\textbf{t}$. \\
		
		We next claim that $C'$ is a randomly generated clause from the distribution $U_k$ if $C$ was drawn from $U_{k,0}$ and is a randomly generated clause from the distribution $Q_{\boldsymbol{\sigma}}$ if $C$ was drawn from $Q_{\boldsymbol{\sigma},0}$. By our construction, the label of the clause $\textbf{y}$ is chosen uniformly at random. Note that choosing a clause uniformly at random from $U_{k,0}$ is equivalent to first uniformly choosing a $k$-tuple of unnegated literals and then choosing a negation pattern for the literals uniformly at random. It is clear that a clause is still uniformly random after adding another negation pattern if it was uniformly random before. Hence, if the original clause $C$ was drawn to the uniform distribution $U_{k,0}$, then $C'$ is distributed according to $U_k$. Similarly, choosing a clause uniformly at random from $Q_{\boldsymbol{\sigma},\mathbf{y}}$ for some $\textbf{y}$ is equivalent to first uniformly choosing a $k$-tuple of unnegated literals and then choosing a negation pattern uniformly at random which makes the clause consistent. As the original negation pattern corresponds to a $\textbf{v}$ randomly chosen from the null space of $\textbf{A}$, the final negation pattern on adding $\textbf{v}({\textbf{y}})$ corresponds to the negation pattern for a uniformly random chosen solution of ${\textbf{y}=\textbf{Av} \mod 2}$ for the chosen $\textbf{y}$. Therefore, the clause $C'$ is a uniformly random chosen clause from $Q_{\boldsymbol{\sigma},y}$ if $C$ is a uniformly random chosen clause from $Q_{\boldsymbol{\sigma},0}$. \\
		
		Hence if it is possible to distinguish $U_k$ and $Q_{\boldsymbol{\sigma}}^{\eta}$ for some randomly chosen $\boldsymbol{\sigma} \in \{0,1\}^n$ with success probability at least $2/3$ in time $t(n)$ with $s(n)$ clauses, then it is possible to distinguish between $U_{k,0}$ and $Q_{\boldsymbol{\sigma},0}^{\eta}$ for some randomly chosen $\boldsymbol{\sigma} \in \{0,1\}^n$ with success probability at least $2/3$ in time $t(n)+s(n)O(k)$ with $s(n)$ clauses.
	\end{proof}
	
	\subsection{Proof of Lemma \ref{hardness_L'}}
	
	\cspredii*
	
	\begin{proof}
		
		Define $E$ to be the event that a clause generated from the distribution $Q_{\boldsymbol{\sigma}}$ of the CSP $\mathcal{C}$ has the property that for all $i$ the $i$th literal belongs to the set $\mathcal{X}_i$, we also refer to this property of the clause as $E$ for notational ease. It's easy to verify that the probability of the event $E$ is $1/k^{k}$. We claim that conditioned on the event $E$, the CSP $\mathcal{C}$ and $\mathcal{C}'$ are equivalent.\\
		
		This is verified as follows. Note that for all $\textbf{y}$, $Q_{\boldsymbol{\sigma},y}$ and $Q'_{\boldsymbol{\sigma},y}$ are uniform on all consistent clauses. Let $\mathcal{U}$ be the set of all clauses with non-zero probability under $Q_{\boldsymbol{\sigma},y}$ and $\mathcal{U}'$ be the set of all clauses with non-zero probability under $Q'_{\boldsymbol{\sigma},y}$. Furthermore, for any $\textbf{v}$ which satisfies the constraint that $\textbf{y}={\textbf{Av} \mod 2}$, let $\mathcal{U}(\textbf{v})$ be the set of clauses $C\in \mathcal{U}$ such that  $\boldsymbol{\sigma}(\text{C})=\textbf{v}$. Similarly, let $\mathcal{U}'(\textbf{v})$ be the set of clauses $C\in \mathcal{U}'$ such that  $\boldsymbol{\sigma}(\text{C})=\textbf{v}$. Note that the subset of clauses in $\mathcal{U}(\textbf{v})$ which satisfy $E$ is the same as the set $\mathcal{U}'(\textbf{v})$. As this holds for every consistent $\textbf{v}$ and the distributions $Q'_{\boldsymbol{\sigma},y}$ and $Q_{\boldsymbol{\sigma},y}$ are uniform on all consistent clauses, the distribution of clauses from $Q_{\boldsymbol{\sigma}}$ is identical to the distribution of clauses $Q'_{\boldsymbol{\sigma}}$ conditioned on the event $E$. The equivalence of $U_k$ and $U'_k$ conditioned on $E$ also follows from the same argument. \\
		
		Note that as the $k$-tuples in $\mathcal{C}$ are chosen uniformly at random from satisfying $k$-tuples, with high probability there are $s(n)$ tuples having property $E$ if there are $O(k^ks(n))$ clauses in $\mathcal{C}$. As the problems $L$ and $L'$ are equivalent conditioned on event $E$, if $L'$ can be solved in time $t(n)$ with $s(n)$ clauses, then $L$ can be solved in time $t(n)$ with $O(k^ks(n))$ clauses. From Lemma \ref{csp_red} and Conjecture 1, $L$ cannot be solved in polynomial time with less than $\tilde{\Omega}(n^{\gamma k/2})$ clauses. Hence $L'$ cannot be solved in polynomial time with less than $\tilde{\Omega}(n^{\gamma k/2}/k^k)$ clauses. As $k$ is a constant with respect to $n$, $L'$ cannot be solved in polynomial time with less than $\tilde{\Omega}(n^{\gamma k/2})$ clauses.
	\end{proof}
\section{Proof of Lower Bound for Small Alphabets}

\subsection{Proof of Lemma \ref{lem:binary_hardness}}

\parityA*
\begin{proof}
%	Suppose $\textbf{A} \in \{0,1\}^{m \times n}$ is chosen at random with each entry being i.i.d. with its distribution uniform on $\{0,1\}$. Let $\mathbf{A}'$ be the sub-matrix of $\textbf{A}$ corresponding to the first $2n/3$ columns and all the $m$ rows. Recall that $\mathcal{S}$ is the set of all $(m \times n)$ matrices $\textbf{A}$ such that the sub-matrix $\mathbf{A}'$ is full row-rank. We claim that $P(\mathbf{A} \in \mathcal{S}) \ge 1-m2^{-n/6}$. To verify, consider the addition of each row one by one to $\textbf{A}'$. The probability of the $i$th row being linearly dependent on the previous $(i-1)$ rows is $2^{i-1-2n/3}$. Hence by a union bound, $\textbf{A}'$ is full row-rank with failure probability at most $m 2^{m-2n/3} \le m2^{-n/6}$. From Definition \ref{lpn} and a union bound over all the $m\le n/2$ parities, any algorithm that can distinguish the outputs from the model $\mathcal{M}(\textbf{A})$ for uniformly chosen $\textbf{A}$ from the distribution over random examples $U_n$ with probability at least $(1-1/(2n))$ over the choice of $\textbf{A}$ needs ${f(n)}$ time or examples. As $P(\mathbf{A} \in \mathcal{S}) \ge 1-m2^{-n/6}$ for a uniformly randomly chosen $\mathbf{A}$, with probability at least $(1-1/(2n)-m2^{-n/6}) \ge (1-1/n)$ over the choice $\textbf{A}\in \mathcal{S}$ any algorithm that can distinguish the outputs from the model $\mathcal{M}(\textbf{A})$ from the distribution over random examples $U_n$   with success probability greater than $2/3$ over the randomness of the examples and the algorithm needs ${f(n)}$ time or examples.
	Suppose $\textbf{A} \in \{0,1\}^{m \times n}$ is chosen at random with each entry being i.i.d. with its distribution uniform on $\{0,1\}$.  Recall that $\mathcal{S}$ is the set of all $(m \times n)$ matrices $\textbf{A}$ which are full row rank. We claim that $P(\mathbf{A} \in \mathcal{S}) \ge 1-m2^{-n/6}$. To verify, consider the addition of each row one by one to $\textbf{A}'$. The probability of the $i$th row being linearly dependent on the previous $(i-1)$ rows is $2^{i-1-n}$. Hence by a union bound, $\textbf{A}'$ is full row-rank with failure probability at most $m 2^{m-n} \le m2^{-n/2}$. From Definition \ref{lpn} and a union bound over all the $m\le n/2$ parities, any algorithm that can distinguish the outputs from the model $\mathcal{M}(\textbf{A})$ for uniformly chosen $\textbf{A}$ from the distribution over random examples $U_n$ with probability at least $(1-1/(2n))$ over the choice of $\textbf{A}$ needs ${f(n)}$ time or examples. As $P(\mathbf{A} \in \mathcal{S}) \ge 1-m2^{-n/2}$ for a uniformly randomly chosen $\mathbf{A}$, with probability at least $(1-1/(2n)-m2^{-n/2}) \ge (1-1/n)$ over the choice $\textbf{A}\in \mathcal{S}$ any algorithm that can distinguish the outputs from the model $\mathcal{M}(\textbf{A})$ from the distribution over random examples $U_n$   with success probability greater than $2/3$ over the randomness of the examples and the algorithm needs ${f(n)}$ time or examples.
\end{proof}

\subsection{Proof of Proposition \ref{binary}}

\begin{restatable}{proposition}{binary}\label{binary}
	With $f(T)$ as defined in Definition \ref{lpn}, for all sufficiently large $T$ and	$1/T^c<\epsilon \le 0.1$ for some fixed constant $c$, there exists a family of HMMs with $T$ hidden states such that any algorithm that achieves average relative zero-one loss, average $\ell_1$ loss, or average KL loss less than $\epsilon$ with probability greater than 2/3 for a randomly chosen HMM in the family needs, requires ${f}(\Omega(\log T/\epsilon))$ time or samples samples from the HMM.
	\end{restatable}

\begin{proof}
	We describe how to choose the family of sequential models $\mathbf{A}_{m\times n}$ for each value of $\epsilon$ and $T$. Recall that the HMM has $T=2^m(2n+m)+m$ hidden states. Let $T'=2^{m+2}(n+m)$. Note that $T'\ge T$. Let $t = \log T'$. We choose $m=t-\log(1/\epsilon)-\log (t/5)$, and $n$ to be the solution of $t=m+\log(n+m)+2$, hence $n=t/(5\epsilon) - m-2$. Note that for $\epsilon\le 0.1$, $n\ge m$. Let $\epsilon'=\frac{2}{9}\frac{m}{n+m}$. We claim $\epsilon\le \epsilon'$. To verify, note that $n+m=t/(5\epsilon)-2$. Therefore,
	\[
	\epsilon'=\frac{2m}{9(n+m)} = \frac{10\epsilon(t-\log(1/\epsilon)-\log (t/5))}{9t(1-10\epsilon/t)} \ge \epsilon,
	\]
	for sufficiently large $t$ and $\epsilon\ge 2^{-ct}$ for a fixed constant $c$. Hence proving hardness for obtaining error $\epsilon'$ implies hardness for obtaining error $\epsilon$. We choose the matrix $\textbf{A}_{m \times n}$ as outlined earlier. The family is defined by the model $\mathcal{M}(\mathbf{A}_{m\times n})$ defined previously with the matrix $\mathbf{A}_{m\times n}$ chosen uniformly at random from the set $\mathcal{S}$.\\
	
	Let $\rho_{01}(\mathcal{A})$ be the average zero-one loss of some algorithm $\mathcal{A}$ for the output time steps $n$ through $(n+m-1)$ and $\delta_{01}'(\mathcal{A})$ be the average relative zero-one loss of $\mathcal{A}$ for the output time steps $n$ through $(n+m-1)$ with respect to the optimal predictions. For the distribution $U_n$ it is not possible to get $\rho_{01}(\mathcal{A}) <0.5$ as the clauses and the label $\textbf{y}$ are independent and $\textbf{y}$ is chosen uniformly at random from $\{0,1\}^m$. For $Q_{\textbf{s}}^{\eta}$ it is information theoretically possible to get $\rho_{01}(\mathcal{A})=\eta/2$. Hence any algorithm which gets error $\rho_{01}(\mathcal{A}) \le 2/5$ can be used to distinguish between $U_n$ and $Q_{\textbf{s}}^{\eta}$. Therefore by Lemma \ref{lem:binary_hardness} any algorithm which gets $\rho_{01}(\mathcal{A})\le 2/5$ with probability greater than $2/3$ over the choice of $\mathcal{M}(\textbf{A})$ needs at least $f(n)$ time or samples. Note that $\delta_{01}'(\mathcal{A})=\rho_{01}(\mathcal{A})-\eta/2$. As the optimal predictor $\mathcal{P}_{\infty}$ gets $\rho_{01}(\mathcal{P}_{\infty})=\eta/2<0.05$, therefore  $\delta_{01}'(\mathcal{A})\le 1/3 \implies \rho_{01}(\mathcal{A})\le 2/5$. Note that $\delta_{01}(\mathcal{A}) \ge \delta_{01}'(\mathcal{A})\frac{m}{n+m}$. This is because $\delta_{01}(\mathcal{A})$ is the average error for all $(n+m)$ time steps, and the contribution to the error from time steps $0$ to $(n-1)$ is non-negative. Also, $\frac{1}{3} \frac{m}{n+m} > {\epsilon'}$, therefore, $\delta_{01}(\mathcal{A})  < {\epsilon'} \implies \delta_{01}'(\mathcal{A})< \frac{1}{3} \implies \rho_{01}(\mathcal{A}) \le 2/5$. Hence any algorithm which gets average relative zero-one loss less than $ {\epsilon'}$ with probability greater than $2/3$ over the choice of $\mathcal{M}(\textbf{A})$ needs $f(n)$ time or samples. The result for $\ell_1$ loss follows directly from the result for relative zero-one loss, we next consider the KL loss.\\
	
	Let $\delta_{KL}'(\mathcal{A})$ be the average KL error of the algorithm $\mathcal{A}$ from time steps $n$ through $(n+m-1)$. By application of Jensen's inequality and Pinsker's inequality, $\delta_{KL}'(\mathcal{A}) \le 2/9 \implies \delta_{01}'(\mathcal{A}) \le1/3$. Therefore, by our previous argument any algorithm which gets  $\delta_{KL}' (\mathcal{A})<2/9$ needs $f(n)$ samples. But as before, $\delta_{KL} (\mathcal{A})\le{\epsilon'} \implies \delta'_{KL} (\mathcal{A})\le 2/9$.  Hence any algorithm which gets average KL loss less than ${\epsilon'}$ needs $f(n)$ time or samples.\\
	
	We lower bound $n$ by a linear function of $\log T/\epsilon$ to express the result directly in terms of $\log T/\epsilon$. We claim that $\log T/\epsilon$ is at most $10n$. This follows because--
	\begin{align}
	\log T/\epsilon \le t/\epsilon = 5(n+m) +10\le 15n \nonumber
	\end{align}
	Hence any algorithm needs ${f}(\Omega( \log T/\epsilon))$ samples and time to get average relative zero-one loss, $\ell_1$ loss, or KL loss less than $\epsilon$  with probability greater than $2/3$ over the choice of $\mathcal{M}(\textbf{A})$.\\
	\iffalse
	Note that $I[\mathcal{M}(\mathbf{A})]$ is at most $m$. To finish, we verify that $ I(\mathcal{M}(\textbf{A}))$ satisfies the required lower bound. We claim that $ I(\mathcal{M}(\textbf{A})) \ge t/20$. Let $\ell=I(\mathcal{M}(\textbf{A}))/\epsilon$. By Proposition \ref{prop:upperbnd}, the predictor $\mathcal{P}_{\ell}$ gets KL error at most $\epsilon$. As noted earlier, $\delta_{KL}(\mathcal{P}_{\ell}) \le\epsilon \implies \delta'_{KL}(\mathcal{P}_{\ell})  \le 2/9 \implies \delta'_{01}(\mathcal{P}_{\ell})  \le 1/3$. Note that it is not information theoretically possible to get $\delta_{01}'<1/3$ with windows smaller than $n/3$ as only looking at the outputs from time $2n/3$ to $n-1$ gives us no information about the future outputs. This is because the matrix $\mathbf{A}'$, which is the sub-matrix of $\mathbf{A}$ corresponding to the first $2n/3$ columns and all rows, is full row-rank, hence the dependence of the parity bits on the inputs from time 0 to $2n/3-1$ cannot be inferred from the outputs after time $2n/3$. This is the reason that we required the submatrix $\mathbf{A}'$ of $\mathbf{A}$ to be full row rank. Therefore $\ell\ge n/3$. We now lower bound $\epsilon n$ as follows--
	\begin{align}
	9\epsilon/2 &\ge \frac{t}{n+t}\nonumber\\
	\implies \epsilon n &\ge t({2}/{9}-\epsilon)\nonumber\\
	\implies \epsilon n &\ge 0.1t\nonumber
	\end{align}
	Hence $I(\mathcal{M}(\textbf{A}))$ is at least $n\epsilon/3\ge t/30$. Hence the model $\mathcal{M}(\textbf{A})$ satisfies $c t \le I(\mathcal{M}(\textbf{A})) \le t $ for some constant $c$.
	\fi
\end{proof}
\section{Proof of Information Theoretic Lower Bound}\label{sec:info_bnd}

\begin{restatable}{proposition}{infobnd}\label{info_bound}
	There is an absolute constant $c$ such that for all $0<\epsilon<0.5$ and sufficiently large $n$, there exists an HMM with $n$ states such that it is not information theoretically possible to get average relative zero-one loss or $\ell_1$ loss less than $\epsilon$ using windows of length smaller than $c\log n/\epsilon^2$, and KL loss less than $\epsilon$ using windows of length smaller than $c\log n/\epsilon$.
\end{restatable}
\begin{proof}
	Consider a Hidden Markov Model with the Markov chain being a permutation on $n$ states. The output alphabet of each hidden state is binary. Each state $i$ is marked with a label $l_i$ which is 0 or 1, let $G(i)$ be mapping from hidden state $h_i$ to its label $l_i$. All the states labeled 1 emit 1 with probability $(0.5+\epsilon)$ and 0 with probability $(0.5-\epsilon)$. Similarly, all the states labeled 0 emit 0 with probability $(0.5+\epsilon)$ and 1 with probability $(0.5-\epsilon)$. Fig. \ref{hmm_fig} illustrates the construction and provides the high-level proof idea.\\
	
	\begin{figure}[h]
		\centering
		\includegraphics[width=2 in]{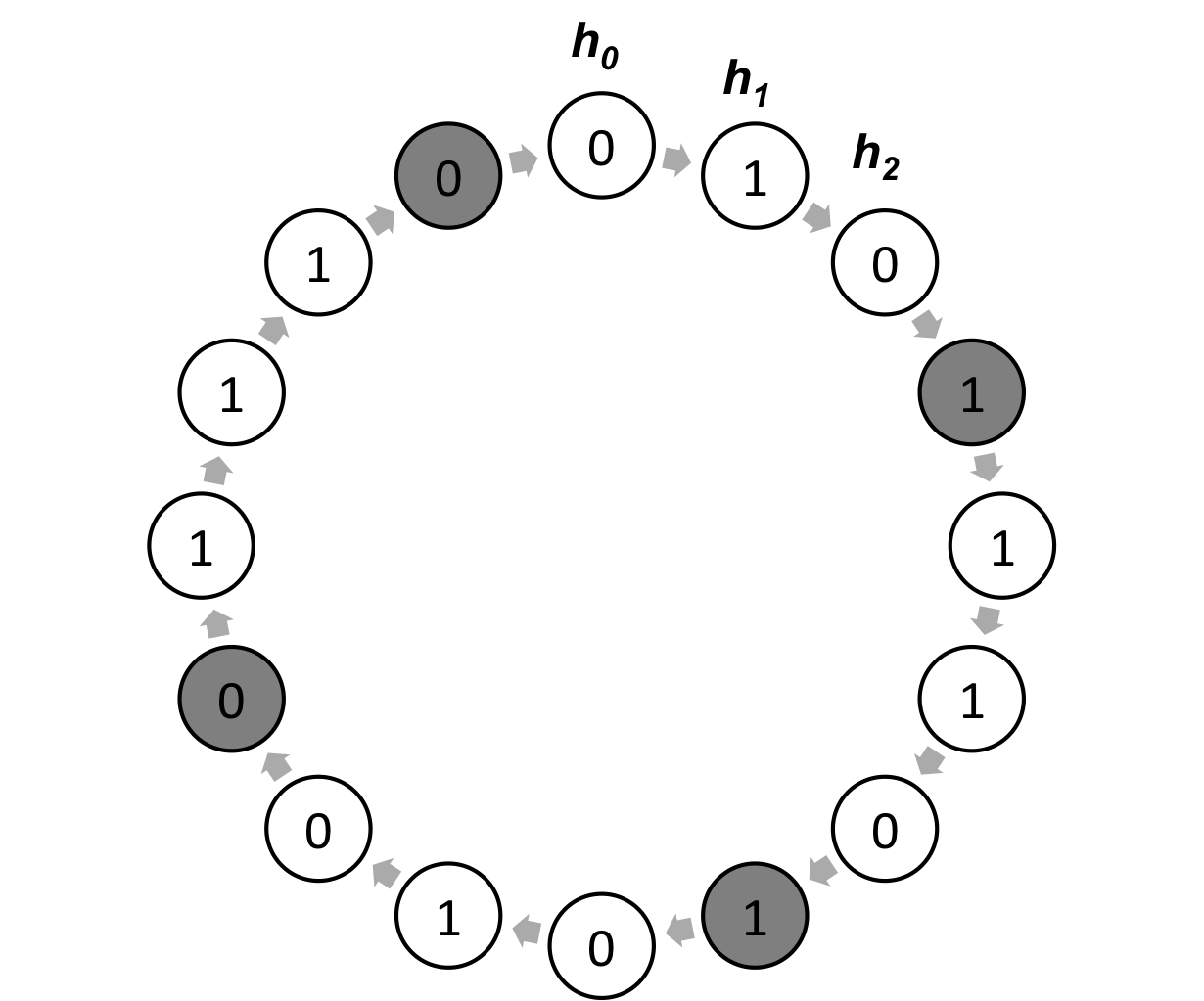}
		\caption{Lower bound construction, $\ell=3, n=16$. A note on notation used in the rest of the proof with respect to this example: $r(0)$ corresponds to the label of $h_0$, $h_1$ and $h_2$ and is $(0,1,0)$ in this case. Similarly, $r(1)=(1,1,0)$ in this case. The segments between the shaded nodes comprise the set $\mathcal{S}_1$ and are the possible sequences of states from which the last $\ell=3$ outputs could have come. The shaded nodes correspond to the states in $\mathcal{S}_2$, and are the possible predictions for the next time step. In this example $\mathcal{S}_1=\{(0,1,0),(1,1,0),(0,1,0),(1,1,1)\}$ and $\mathcal{S}_2= \{1,1,0,0\}$.}
		\label{hmm_fig}
	\end{figure}
	
	Assume $n$ is a multiple of $(\ell+1)$, where $(\ell+1)=c\log n/\epsilon^2$, for a constant $c=1/33$. We will regard $\epsilon$ as a constant with respect to $n$. Let $n/(\ell+1)=t$. We refer to the hidden states by $h_i$, where$ 0 \le i \le (n-1)$, and $h_i^j$ refers to the sequence of hidden states $i$ through $j$. We will show that a model looking at only the past $\ell$ outputs cannot get average zero-one loss less than $0.5-o(1)$. As the optimal prediction looking at all past outputs gets average zero-one loss $0.5-\epsilon+o(1)$ (as the hidden state at each time step can be determined to an arbitrarily high probability if we are allowed to look at an arbitrarily long past), this proves that windows of length $\ell$ do not suffice to get average zero-one error less than $\epsilon-o(1)$ with respect to the optimal predictions. Note that the Bayes optimal prediction at time $(\ell+1)$ to minimize the expected zero-one loss given outputs from time $1$ to $\ell$ is to predict the mode of the distribution $\Pr(x_{\ell+1}|x_{1}^{\ell}=s_1^{\ell})$ where $s_1^{\ell}$ is the sequence of outputs from time $1$ to $\ell$. Also, note that $\Pr(x_{\ell+1}|x_{1}^{\ell}=s_1^{\ell})= \sum _i \Pr(h_{i_{\ell}=i}|x_{1}^{\ell}=s_1^{\ell})\Pr(x_{\ell+1}|h_{i_{\ell}=i})$ where $h_{i_{\ell}}$ is the hidden state at time $\ell$. Hence the predictor is a weighted average of the prediction of each hidden state with the weight being the probability of being at that hidden state.\\
	
	We index each state $h_i$ of the permutation by a tuple $(f(i),g(i))=(j,k)$ where ${j = i \mod (\ell+1)}$ and $k=\lfloor \frac{i}{\ell+1}\rfloor$ hence $0 \le j \le \ell$, $0\le k\le (t-1)$ and $i=k(\ell+1)+j$. We help the predictor to make the prediction at time $(\ell+1)$ by providing it with the index $f(i_{\ell})= i_{\ell} \mod (\ell+1)$ of the true hidden state $h_{i_{\ell}}$ at time $\ell$. Hence this narrows down the set of possible hidden states at time $\ell$ (in Fig. \ref{hmm_fig}, the set of possible states given this side information are all the hidden states before the shaded states). The Bayes optimal prediction at time $(\ell+1)$ given outputs $s_1^{\ell}$ from time $1$ to $\ell$ and index $f(h_{i_{\ell}})=j$ is to predict the mode of $\Pr(x_{\ell+1}|x_{1}^{\ell}=s_1^{\ell}, f(h_{i_{\ell}})=j)$. Note that by the definition of Bayes optimality, the average zero-one loss of the prediction using $\Pr(x_{\ell+1}|x_{1}^{\ell}=s_1^{\ell}, f(h_{i_{\ell}})=j)$ cannot be worse than the average zero-one loss of the prediction using $\Pr(x_{\ell+1}|x_{1}^{\ell}=s_1^{\ell})$. Hence we only need to show that the predictor with access to this side information is poor. We refer to this predictor using $\Pr(x_{\ell+1}|x_{1}^{\ell}=s_1^{\ell}, f(h_{i_{\ell}})=j)$ as $\mathcal{P}$. We will now show that there exists some permutation for which the average zero-one loss of the predictor $\mathcal{P}$ is $0.5 - o(1)$. We argue this using the probabilistic method. We choose a permutation uniformly at random from the set of all permutations. We show that the expected average zero-one loss of the predictor $\mathcal{P}$ over the randomness in choosing the permutation is $0.5 - o(1)$. This means that there must exist some permutation such that the average zero-one loss of the predictor $\mathcal{P}$ on that permutation is $0.5 - o(1)$.\\

	To find the expected average zero-one loss of the predictor $\mathcal{P}$ over the randomness in choosing the permutation, we will find the expected average zero-one loss of the predictor $\mathcal{P}$ given that we are in some state $h_{i_{\ell}}$ at time $\ell$. Without loss of generality let $f(i_{\ell})=0$ and $g(i_{\ell})=(\ell-1)$, hence we were at the $(\ell-1)$th hidden state at time $\ell$. Fix any sequence of labels for the hidden states $h_0^{\ell-1}$. For any string $s_0^{\ell-1}$ emitted by the hidden states $h_0^{\ell-1}$ from time 0 to $\ell-1$, let $\E[\delta(s_0^{\ell-1})]$ be the expected average zero-one error of the predictor $\mathcal{P}$ over the randomness in the rest of the permutation. Also, let $\E[\delta(h_{\ell-1})]=\sum_{s_0^{\ell-1}}^{}\E[\delta(s_0^{\ell-1})]\Pr[s_0^{\ell-1}]$ be the expected error averaged across all outputs. We will argue that $\E[\delta(h_{\ell-1})]=0.5-o(1)$. The set of hidden states $h_i$ with $g(i)=k$ defines a segment of the permutation, let $r(k)$ be the label $G(h_{(k-1)({\ell}+1)}^{k(\ell+1)-2})$ of the segment $k$, excluding its last bit which corresponds to the predictions. Let $\mathcal{S}_1=\{r(k), \forall\; k \ne 0\}$ be the set of all the labels excluding the first label $r(0)$ and $\mathcal{S}_2=\{G(h_{k(\ell+1)+\ell}), \forall\; k\}$ be the set of all the predicted bits (refer to Fig. \ref{hmm_fig} for an example). Consider any assignment of $r(0)$. To begin, we show that with high probability over the output $s_0^{\ell-1}$, the Hamming distance $D(s_0^{\ell-1},r(0))$ of the output $s_0^{\ell-1}$ of the set of hidden states $h_0^{\ell-1}$ from $r(0)$ is at least $\frac{\ell}{2}-2\epsilon \ell$. This follows directly from Hoeffding's inequality\footnote{For $n$ independent random variables $\{X_i\}$ lying in the interval $[0,1]$ with $\bar{X}=\frac{1}{n}\sum_{i}^{}X_i$,  ${\Pr[X\le\E[\bar{X}] -t]}\le e^{-2nt^2}$. In our case $t=\epsilon$ and $n=\ell$.} as all the outputs are independent conditioned on the hidden state--
	\begin{align}
		\Pr[D(s_0^{\ell-1},r(0))\le {\ell}/{2}-2\epsilon \ell] &\le e^{-2\ell\epsilon^2}
		\le n^{-2c}\label{hoeffding}
	\end{align}
	We now show that for any $k\ne 0$, with decent probability the label $r(k)$ of the segment $k$ is closer in Hamming distance to the output $s_0^{\ell-1}$ than $r(0)$. Then we argue that with high  probability there are many such segments which are closer to $s_0^{\ell-1}$ in Hamming distance than $r(0)$. Hence these other segments are assigned as much weight in predicting the next output as $r(0)$, which means that the output cannot be predicted with a high accuracy as the output bits corresponding to different segments are independent.\\
	
	We first find the probability that the segment corresponding to some $k$ with label $r(k)$ has a Hamming distance less than $\frac{\ell}{2}-\sqrt{\ell\log t/8}$ from any fixed binary string $x$ of length $\ell$. Let $F(l,m,p)$ be the probability of getting at least $l$ heads in $m$ i.i.d. trails with each trial having probability $p$ of giving a head. $F(l,m,p)$ can be bounded below by the following standard inequality--
	\begin{align}
		F(l,m,p) \ge \frac{1}{\sqrt{2m}}\exp\Big(-mD_{KL}\Big(\frac{l}{m}\Big\| p\Big)\Big)\nonumber
	\end{align}
	where $D_{KL}(q\parallel p)=q\log\frac{q}{p}+(1-q)\log\frac{1-q}{1-p}$. We can use this to lower bound $\Pr\Big[D(r(k),x) \le {\ell}/{2}-\sqrt{\ell\log t/8}\Big] $,
	\begin{align}
		\Pr\Big[D(r(k),x) \le {\ell}/{2}-\sqrt{\ell\log t/8}\Big] &= F({\ell}/{2}+\sqrt{\ell\log t/8},\ell,1/2)\nonumber\\
		&\ge \frac{1}{\sqrt{2{\ell}}}\exp\Big(-\ell D_{KL}\Big(\frac{1}{2}+\sqrt{\frac{\log t}{8\ell}}\Big\| \frac{1}{2}\Big)\Big)\nonumber
	\end{align}
	Note that $D_{KL}(\frac{1}{2}+v\parallel \frac{1}{2}) \le 4v^2$ by using the inequality $\log (1+v)\le v$. We can simplify the KL-divergence using this and write--
	\begin{align}
		\Pr\Big[D(r(k),x) \le {\ell}/{2}-\sqrt{\ell\log t/8}\Big] \ge {1}/{\sqrt{2\ell t}}\label{eq:bound1}
	\end{align}
Let $\mathcal{D}$ be the set of all $k\ne 0$ such that $D(r(k),x) \le \frac{\ell}{2}- \sqrt{\ell\log t/8}$ for some fixed $x$. We argue that with high probability over the randomness of the permutation $|\mathcal{D}|$ is large. This follows from Eq. \ref{eq:bound1} and the Chernoff bound\footnote{For independent random variables $\{X_i\}$ lying in the interval $[0,1]$ with $X=\sum_{i}^{}X_i$ and $\mu=\E[X]$, ${\Pr[X\le (1-\epsilon)\mu]}\le \exp({-{\epsilon^2\mu}/{2}})$. In our case $\epsilon=1/2$ and $\mu=\sqrt{t/(2\ell)}$.} as the labels for all segments $r(k)$ are chosen independently--
\begin{align}
	\Pr\Big[|\mathcal{D}| \le \sqrt{{t}/({8\ell})}\Big] \le e^{-\frac{1}{8}\sqrt{{t}/(2\ell)}}\nonumber
\end{align}
Note that $\sqrt{{t}/({8\ell})} \ge n^{0.25}$. Therefore for any fixed $x$, with probability $1-\exp(-\frac{1}{8}\sqrt{\frac{t}{2\ell}})\ge 1-n^{-0.25}$ there are $\sqrt{\frac{t}{8\ell}}\ge n^{0.25}$ segments in a randomly chosen permutation which have Hamming distance less than ${\ell}/{2}-\sqrt{\ell\log t/8}$ from $x$. Note that by our construction $2\epsilon \ell \le  \sqrt{\ell\log t/8}$ because $\log (\ell +1) \le (1-32c)\log n $. Hence the segments in $\mathcal{D}$ are closer in Hamming distance to the output $s_0^{\ell-1}$ if $D(s_0^{\ell-1},r(0))> {\ell}/{2}-2\epsilon \ell$. \\

Therefore if $D(s_0^{\ell-1},r(0))> {\ell}/{2}-2\epsilon \ell$, then with high probability over randomly choosing the segments $\mathcal{S}_1$ there is a subset $\mathcal{D}$ of segments in $\mathcal{S}_1$ with $|\mathcal{D}|\ge n^{0.25}$ such that all of the segments in $\mathcal{D}$ have Hamming distance less than $D(s_0^{\ell-1},r(0))$ from $s_0^{\ell-1}$. Pick any $s_0^{\ell-1}$ such that $D(s_0^{\ell-1},r(0))> {\ell}/{2}-2\epsilon \ell$. Consider any set of segments $\mathcal{S}_1$ which has such a subset $\mathcal{D}$ with respect to the string $s_0^{\ell-1}$. For all such permutations, the predictor $\mathcal{P}$ places at least as much weight on the hidden states $h_i$ with $g(i)=k$, with $k$ such that $r(k) \in \mathcal{D}$ as the true hidden state $h_{\ell-1}$. The prediction for any hidden state $h_i$ is the corresponding bit in $\mathcal{S}_2$. 
Notice that the bits in $\mathcal{S}_2$ are independent and uniform as we've not used them in any argument so far. The average correlation of an equally weighted average of $m$ independent and uniform random bits with any one of the random bits is at most $1/\sqrt{m}$. Hence over the randomness of $\mathcal{S}_2$, the expected zero-one loss of the predictor is at least $0.5-n^{-0.1}$. Hence we can write-
\begin{align}
	\E[\delta(s_0^{\ell-1})] &\ge (0.5 - n^{-0.1})\Pr[|\mathcal{D}| \ge \sqrt{{t}/({8\ell})}]\nonumber\\
	&\ge (0.5 - n^{-0.1}) (1-e^{-n^{0.25}})\nonumber\\
	&\ge 0.5-2n^{-0.1}\nonumber
\end{align}
By using Equation \ref{hoeffding}, for any assignment $r(0)$ to $h_0^{\ell-1}$
\begin{align}
\E[\delta(h_{\ell-1})] &\ge \Pr\Big[D(s_0^{\ell-1},r(0)) > {\ell}/{2}-2\epsilon \ell\Big] E\Big[\delta(s_0^{\ell-1})\Big|D(s_0^{\ell-1},r(0)) > {\ell}/{2}-2\epsilon \ell\Big]\nonumber\\
&\ge (1-n^{-2c})(0.5-2n^{-0.1})\nonumber\\
&= 0.5-o(1)\nonumber
\end{align}
As this is true for all assignments $r(0)$ to $h_0^{\ell-1}$ and for all choices of hidden states at time $\ell$, using linearity of expectations and averaging over all hidden states, the expected average zero-one loss of the predictor $\mathcal{P}$ over the randomness in choosing the permutation is $0.5 - o(1)$. This means that there must exist some permutation such that the average zero-one loss of the predictor $\mathcal{P}$ on that permutation is $0.5 - o(1)$. Hence there exists an HMM on $n$ states such that is not information theoretically possible to get average zero-one error with respect to the optimal predictions less than $\epsilon-o(1)$ using windows of length smaller than $c\log n/\epsilon^2$ for a fixed constant $c$.\\

Therefore, for all $0<\epsilon<0.5$ and sufficiently large $n$, there exits an HMM with $n$ states such that it is not information theoretically possible to get average relative zero-one loss less than $\epsilon/2<\epsilon-o(1)$ using windows of length smaller than $c\epsilon^{-2}\log n$. The result for relative zero-one loss follows on replacing $\epsilon/2$ by $\epsilon'$ and setting $c'=c/4$. The result follows immediately from this as the expected relative zero-one loss is less than the expected $\ell_1$ loss. For KL-loss we use Pinsker's inequality and Jensen's inequality.
\end{proof}
\section*{Acknowledgements}

Sham Kakade acknowledges funding from the Washington Research Foundation for Innovation in Data-intensive Discovery, and the NSF Award CCF-1637360. Gregory Valiant and Sham Kakade acknowledge funding form NSF Award CCF-1703574. Gregory was also supported by NSF CAREER Award CCF-1351108 and a Sloan Research Fellowship. 

\bibliographystyle{unsrtnat}
\bibliography{all,references}

\end{document}